\documentclass{article}

\usepackage[utf8]{inputenc} 
\usepackage[T1]{fontenc}    
\usepackage[margin=0.85in]{geometry} 
\usepackage{amsmath}
\usepackage{amssymb}
\usepackage{amsthm}
\usepackage{multirow}
\usepackage{graphicx}
\usepackage{hyperref}       
\usepackage{url}            
\usepackage{booktabs}       
\usepackage{amsfonts}       
\usepackage{nicefrac}       
\usepackage{microtype}      
\usepackage[dvipsnames]{xcolor}         
\usepackage{natbib}
\usepackage{etoc}

\usepackage{algorithm}
\usepackage[noend]{algpseudocode}
\title{Proximal basin hopping: global optimization with guarantees}

\author{%
  Guillaume Lauga \\
  LJAD \\
  Université Côte d'Azur \\
  France \\
  \texttt{guillaume.lauga@univ-cotedazur.fr}
  \and
  Cesare Molinari \\
  MALGA \\
  Università di Genova \\
  Italy \\
  \texttt{cecio.molinari@gmail.com}
  \and
  Samuel Vaiter \\
  CNRS, LJAD \\
  Université Côte d'Azur \\
  France \\  
  \texttt{samuel.vaiter@univ-cotedazur.fr}
}
\newcommand{\Id}{\mathrm{Id}}
\newcommand{\RR}{\mathbb{R}}
\newcommand{\prox}{\mathrm{prox}}
\newcommand{\Att}{\mathrm{Att}}

\DeclareMathOperator*{\argmin}{arg\,min}
\newtheorem{lemma}{Lemma}

\newtheorem{proposition}{Proposition}
\newtheorem{assumption}{Assumption}
\newtheorem{theorem}{Theorem}
\newtheorem{fact}{Fact}
\newtheorem{remark}{Remark}
\begin{document}

\maketitle

\begin{abstract}
  Global optimization is a challenging problem, with plenty of algorithms displaying empirical success, but scarce theoretical backing. In this work, we propose a new theoretical framework called \textbf{Proximal Basin Hopping} (PBH), carefully tailored to combine \textit{proximal optimization} and \textit{local minimization}. We use it to construct a practical algorithm that converges to the global minimizer with high probability, when using a \textit{finite amount of samples}. Proximal Basin Hopping outperforms well known algorithms with theoretical backing on standard synthetic hard functions, and real problems such as fitting scaling laws for deep learning. Furthermore, the higher the dimension, the better the performance gap.%
\end{abstract}

\section{Introduction} Finding global minimizers of highly non-convex functions is a challenging but crucial task in various fields such as computational chemistry and biology \citep{prentiss2008protein,alvarez2025review}, engineering design \citep{arora1995global,englander2020hopping} and of course machine learning \citep{goodfellow2016deep}. We are interested in solving the following optimization problem
\begin{equation} \label{eq:problem}
    \mathrm{Find } \, \widehat{x} \in \argmin_{x \in \RR^d} f(x).
\end{equation}
where $f$ is a non-convex, continuous function. We will assume that $f$ possesses one global minimizer and some regularity that can be exploited to compute some descent steps (or find local minimizers directly). To tackle this problem, we propose a hybrid approach that sits at the intersection between proximal algorithms \citep{parikh2014proximal} and basin hopping algorithms \citep{wales1997basinhopping}.  \footnotetext{\textbf{Acknowledgements.} This work has been supported by the French government, through the 3IA Cote d’Azur Investments in the project managed by the National Research Agency (ANR) with the reference number ANR-23-IACL-0001, the ANR project PRC MAD ANR-24-CE23-1529 and the support of the “France 2030” funding ANR-23-PEIA-0004 (PDE-AI).}

Standard first-order algorithms can only hope to obtain local solutions to Problem \eqref{eq:problem}, which has prompted many to propose global optimization techniques, most of which are heuristic and do not possess strong convergence guarantees. Global optimization algorithms perform by combining exploration (through random perturbation of the current guesses) and exploitation steps (through function evaluation and/or local minimization) to reach a solution. The exploration requires a number of samples going to infinity to fit an underlying algorithm with increasing precision \citep{fornasier2024consensus,zhang2024inexact}, or a number of samples growing exponentially fast with the dimension to cover the space sufficiently enough \citep{hansen1992global,bouttier2020regret} to reach acceptable accuracy. 
This becomes rapidly cumbersome, even for problems of (nowadays) relatively small dimension (in the few hundreds) \citep{papenmeier2025understanding}. On the other hand, the inexactness of the exploitation steps through local minimization is difficult to account for in the theoretical analysis \citep{gyorgy2011efficient}. Hence, we propose a new framework opening the path to quantify the effect of sample size and inexactness on the convergence to the global minimizer.
\begin{figure}[t] \centering
    \includegraphics[trim={0em 0em 0em 2em}, clip, width=0.35\textwidth]{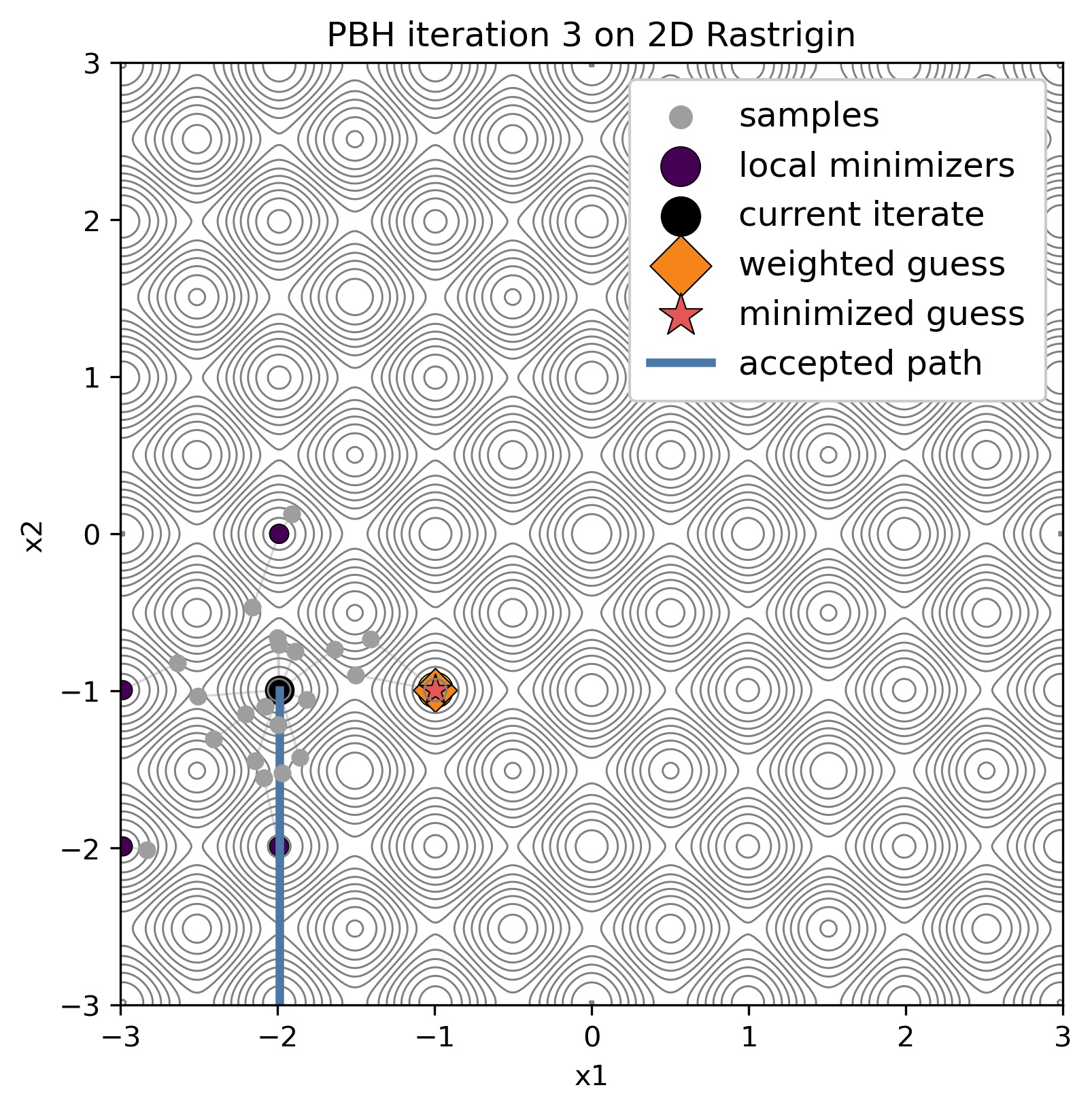}
    \hspace{1em}
    \includegraphics[trim={0em 0em 0em 2em}, clip, width=0.35\textwidth]{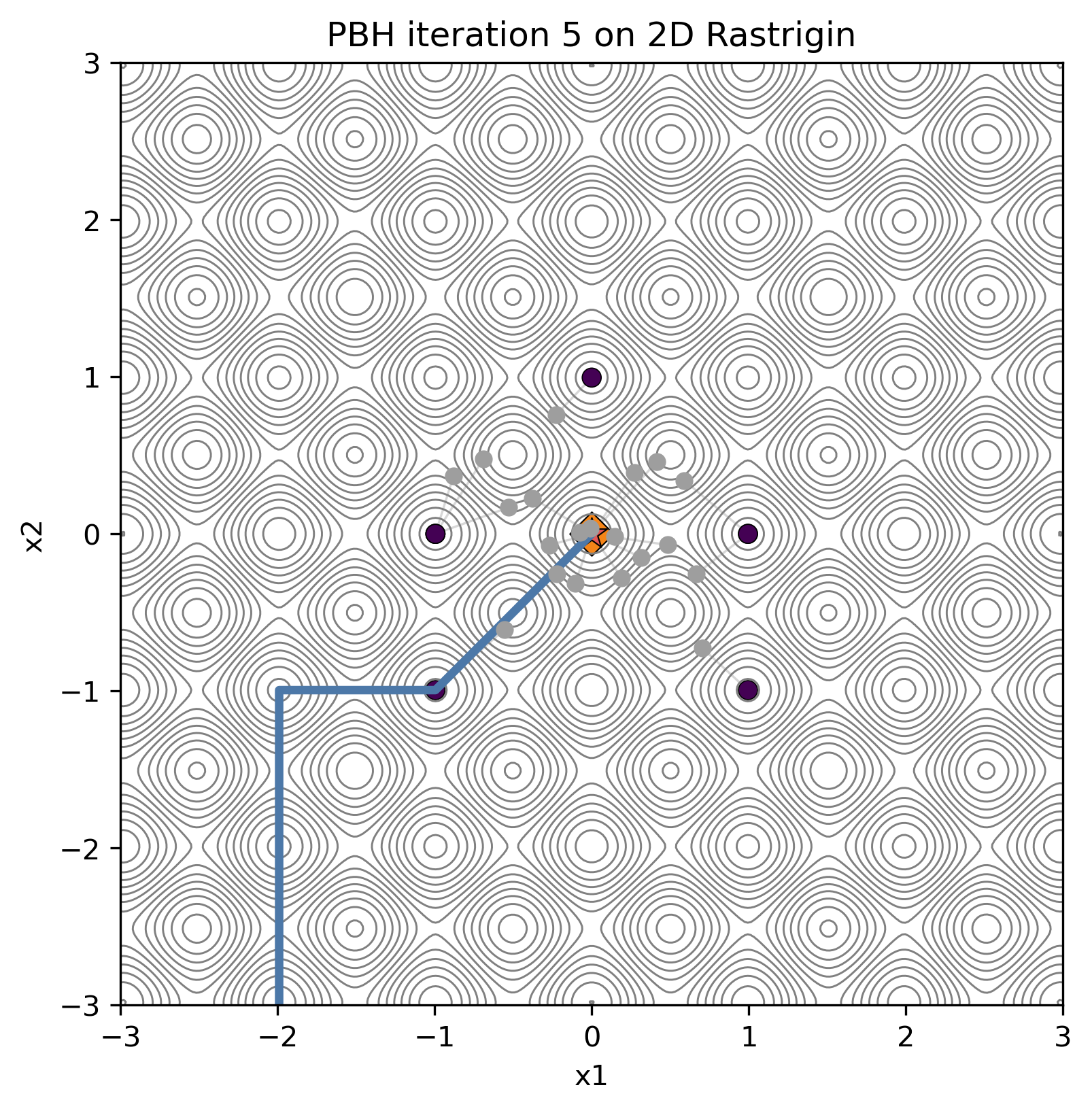}
    \caption{Proximal Basin Hopping on the 2D-Rastrigin function at iteration 3 (left) and iteration 5 (right), where convergence to the global minimizer, $x_\ast = (0,0)$, has occurred. From the current iterate, we draw several samples, that are then sent to their nearest achievable local minimizer, then weighted with respect to the achieved function value giving a new guess that is again sent to its nearest achievable local minimizer (if it is not there already).  
    \label{fig:loc_PBH}}
\end{figure}
\paragraph{Contributions.} The method we propose, which we call \textbf{Proximal Basin Hopping (PBH)} is derived from an ideal operator $\prox_{T_f}^\gamma f (x)$,
where $T_f$ is a deterministic local solver, yielding for (almost) any starting point its achievable local minimizer, and $\prox$ denotes a proximal operation. This operator mimics $T_f(\prox_{f \circ T_f}(x))$ from which we gradually relax the required computations, in order to construct an algorithm using a finite amount of samples and possibly inexact local minimization (see Figure \ref{fig:loc_PBH} for an illustration). We obtain thus four operators, from most ideal to most practical:
\begin{enumerate}
    \item Ideal PBH: we make sense of the $\approx$ in Lemma \ref{lm:Tf_prox},
\begin{equation*}
    S_f^\gamma(x) := \prox_{T_f}^\gamma f (x) \approx T_f\!\left(\prox_{\gamma(f\circ T_f)}(x)\right).
\end{equation*}
\item Exact expectation PBH: the proximal operator is replaced by its zeroth-order approximation
\begin{equation*}
  S_f^\gamma(x)\approx_{\delta} T_f\left(\frac{\mathbb{E}_{y \sim \mathcal{N}(x,\delta \gamma\Id)}[T_f(y) \exp(-f\left(T_f(y) \right)/\delta)]}{\mathbb{E}_{y \sim \mathcal{N}(x,\delta \gamma\Id)}[\exp(-f\left(T_f(y)\right)/\delta)]}\right).
\end{equation*}
\item Approximated expectation PBH: the zeroth-order approximation of the $\prox$ is estimated using $N$ samples. For $i=1,\ldots,N, \, y_i \sim \mathcal{N}(x,\delta \gamma\Id)$
\begin{equation*}
  S_f^\gamma(x)\approx_{\delta}^N T_f\left(\frac{\sum_{i=1}^N T_f(y_i) \exp(-f\left(T_f(y_i) \right)/\delta)}{\sum_{i=1}^N \exp(-f\left(T_f(y_i) \right)/\delta)}\right).
\end{equation*}
\item Approximated expectation \& inexact local solving PBH: 
\begin{equation*}
  S_f^\gamma(x)\approx_{\delta,\varepsilon}^N T_f^\varepsilon \left(\frac{\sum_{i=1}^N T_f^\varepsilon(y_i) \exp(-f\left(T_f^\varepsilon(y_i) \right)/\delta)}{\sum_{i=1}^N \exp(-f\left(T_f^\varepsilon(y_i) \right)/\delta)} \right),
\end{equation*}
\end{enumerate}
By allowing inexact computations of the proximal operator and of the local solver, \textbf{we obtain a practical algorithm, that can recover the global minimizer.} This relies on a structural assumption on $f$: \textit{the global minimizer is sufficiently different in function value from local minimizers}. %

We present extensive experiments on hard-to-optimize functions in moderate dimension ($2$ to $200$), and apply our algorithm to the problem of fitting scaling laws on real data from \citep{shukorscaling}. When the dimension is low (between $1$ and $10$) our algorithm is competitive with basin hopping, with respect to CPU time. It becomes conclusively better from moderate to high dimension.

\paragraph{Summary of theoretical results.}
Our goal is to show the convergence of our most practical algorithm. Such a result is obtained building up on the proofs of convergence of more ideal algorithms, by trying to emulate the same behavior: the second operator induces convergence because it approximates the first operator. The third operator induces convergence by emulating the second, and so on.
\begin{itemize}
    \item[1.] Ideal PBH: convergence holds through the properties of the potential $V_x(M):=1/(2\gamma) \mathrm{dist}(x,\overline{\Att(M)})^2 + f(M)$ and of its minimizers which are the elements of $S_f^\gamma(x)$: it is sufficiently well-behaved for finite time convergence to occur.
\item[2.] Exact expectation PBH: by proving convergence of the expectation onto the minimizers of $V_x(\cdot)$ when $\delta$ goes to $0$, we can reuse the convergence arguments of the first algorithm to guarantee convergence to the global minimizer (but not in finite time).
\item[3.] Approximated expectation PBH: first by standard arguments, bringing $N$ to infinity allows us to recover the setting of the second algorithm. In a following analysis, we show that with high probability, we only need a finite $N$ to converge.
\item[4.] Approximated expectation \& inexact local solving PBH: 
we highlight simple conditions on the inexactness that should hold to obtain convergence.
\end{itemize}

\section{Structural assumptions on $f$ and the local solver}
In this section, we present in detail the structural assumptions we make on $f$ and on the local solver $T_f$ in order to guarantee convergence of our algorithms.
We make assumptions on $f$ that are either standard in optimization, or reasonable given our goal of reaching the global minimizer.
\begin{assumption} \label{ass:f_properties1}
    The function $f:\RR^d \to \RR \cup \{+\infty\}$ is proper, continuous, and coercive.
    We define the family of sets of local minimizers of $f$ as
\begin{align}
\mathcal{M}:= \Big\{ M ~\mathrm{connected} : & \exists ~\mathcal{V}~\mathrm{(open)} \subset \RR^d,\ 
 M \subsetneq \mathcal{V}, \nonumber\\
& \ M=\argmin_{x\in \mathcal{V}} f(x),\forall x \in \mathcal{V}\setminus M,\ f(x)>f(M) \Big\}.
\end{align}
We assume that every
$M\in\mathcal M$ is Borel, bounded, and only finitely many elements of $\mathcal M$
intersect each compact subset of $\mathbb R^d$.
\end{assumption}
This definition of $\mathcal{M}$ prevents redundancy, and ensures that the $M$ are pairwise disjoints. Thus, $\mathcal{M}$ consists of possibly infinitely many disjoint, connected components of local minimizers $M\in \RR^d$, but finitely many on any compact set. In other words, $f$ is "tame" (see App. \ref{app:ass_f} for some examples). %
In the basin hopping literature, the standard hypothesis is for $f$ to have finitely many of these components\footnote{Finite does not mean small, the number of minimizers grows exponentially with the dimension for the Lennard-Jones cluster problem, introduced in the seminal paper of basin hopping \citep{wales1997basinhopping}.}. This case is easily covered by our framework. We take the liberty of writing $f(M)<f(Q)$ for any component $Q$ and $M$ in $\mathcal{M}$ as $f$ is constant on each connected component.  We will denote the union of all elements of $\mathcal{M}$ as $\mathcal{M}_\ast = \bigcup_{M \in \mathcal{M}} M$. We also need separation between the global minimizer and other local minima:
\begin{assumption} \label{ass:f_properties3}
    The function $f$ has a unique global minimizer $\left\{x_\ast\right\}$. $\exists ~ \mu >0$ such that for all $M \in \mathcal{M}\setminus \{x_\ast\}$, $f_{\min} + \mu< f(M)$.
\end{assumption}
This assumption is relatively strong by preventing arbitrary accumulation of almost global local minimizers, but also in line with previous assumptions of the literature such as \citep[Assumption 3]{zhang2024inexact} which requires that no critical point $x\neq x_\ast$ such that $f(x)<f_{\min}+\mu$ exists. %
\begin{assumption} \label{ass:Tf}
    The local solver $T_f:\RR^d \to \RR^d$ is measurable, single-valued, and outputs almost everywhere a local minimizer, i.e., $T_f(x) \in \mathcal{M}_\ast$ a.e..
\end{assumption}

$T_f$ maps any point to its unique achievable local minimizer. We define this achievability with the attractors of the local minimizers, which we denote by $\Att$ 
\begin{equation}
    (\forall M \in \mathcal{M}), \quad \mathrm{Att}(M) := \left\{ y \in \RR^d, ~ T_f(y) \in M\right\}.
\end{equation} 
Obviously the geometry of $\Att$ depends a lot on the solver \citep{asenjo2013visualizing,levy2018attraction,levy2026analyzing}, we can operate under the assumption that these sets are open or closed but well-behaved\footnote{See \citep{levy2018attraction,levy2026analyzing} for a more exhaustive discussion on basins of attraction.}. Furthermore, under standard assumptions the attraction sets of critical points which are not local minimizers are of measure $0$ \citep{lee2016gradient}. As a consequence, the sets $\Att(M)$ form a measurable partition of $\RR^d$ up to a null set (Lemma \ref{lm:null_set}).
\begin{assumption} \label{ass:Tf_continuity}
For every $M \in\mathcal M$, the restriction of $T_f$ to $\mathrm{int}(\Att(M))$
is continuous.
\end{assumption}
This assumption is verified for a gradient flow, or a finitely many gradient steps, if the gradient is continuous.
See App. \ref{app:ass_f} for a detailed discussion when $T_f$ is built with gradient descent.

We make one last structural assumption on the behavior of $M \mapsto \mathrm{dist}(x,\overline{\Att(M)})$ for $x \in \RR^d$.
\begin{assumption}\label{ass:Vx_compact_argmin}
For every $x\in\mathbb{R}^d$ and $\gamma>0$, the function
\begin{equation}
V_x:M\in\mathcal{M} \mapsto f(M)+\frac{1}{2\gamma}\mathrm{dist}(x,\overline{\Att(M)})^2
\end{equation}
has a nonempty bounded set of minimizers on $\mathcal{M}$. Furthermore, 
\begin{equation}
    \mathrm{dist}(x,\Att(M)^\circ)=\mathrm{dist}(x,\overline{\Att(M)})=\mathrm{dist}(x,\Att(M)).
\end{equation}
We define:
\begin{equation}
    S_f^\gamma(x) := \prox_{T_f}^\gamma f (x) := \argmin_{M \in \mathcal{M}} V_x(M).
\end{equation}
\end{assumption} 
This assumption is relatively strong (it implies well-posedness of $S_f^\gamma$) but also reasonable: it holds for gradient flow. The coercivity of $f$ yields the coercivity of $V_x$, and the assumption that $f$ has finitely many components of local minimizers on any bounded set imply the boundedness of the $\argmin$. The non-emptiness is always true if $f$ is $C^1$ and $T_f$ consists of finitely many gradient steps as it is continuous, yielding closedness of $\Att(M)$. For a general gradient flow, $\Att(M)$ is open: take for instance $f:=\Vert \cdot \Vert^2/2$. Hence, this assumption. Finally, we define for $x\in \RR^d$, \begin{equation} \label{eq:fact1}
    \alpha_x := \inf_{M \in \mathcal M} V_x(M), \, M = \bigcup_{M \in S_f^\gamma(x)}M.
\end{equation}
\section{Proposed method}
We first consider our ideal operator. Its study is of great interest as it will provide the optimization problem which we hope our relaxed algorithm will emulate with the proper parameter $\delta$. For any entry $x\in\RR^d$, the operator $S_f^\gamma$ identifies nearby local minimizers. All the proofs of this section are deferred to App. \ref{app:proofs_ideal_pbh}.
Under the assumption that $x$ does not see $\Att(M)$ only through its boundary, we can prove that $S_f^\gamma(x) = T_f(\prox_{f \circ T_f}(x))$, which is a helpful characterization in order to approximate $S_f^\gamma$ later.
\begin{lemma} \label{lm:Tf_prox} Suppose that Assumptions \ref{ass:f_properties1}, \ref{ass:Tf}, and \ref{ass:Vx_compact_argmin} hold. Set $\gamma>0$. We have for all $x\in \RR^d$ such that for every
$
M\in S_f^\gamma(x), \, \mathrm{proj}_{\overline{\Att(M)}}(x)\cap \Att(M)\neq\varnothing,
$
that $S_f^\gamma(x) = T_f(\prox_{f \circ T_f}(x))$.
\end{lemma}
This operator $S_f^\gamma$ identifies the achievable minimizers to which $x$ can be sent. The selection of the minimizer depends on the value of $\gamma$. Provided that $\gamma$ is large enough, $S_f^\gamma$ recovers the global minimizer in one iteration.
\begin{lemma} \label{lm:gamma_ast}Suppose that Assumptions \ref{ass:f_properties1}, \ref{ass:f_properties3}, and \ref{ass:Vx_compact_argmin} hold. For all $x\in\RR^d$, there exists $\gamma_x>0$ such that $
        \{x_\ast\} \in S_f^\gamma(x)       
   $ for every $\gamma \geq \gamma_x$.
\end{lemma}
As we will see, this relationship between $\gamma$ and the elements of $S_f^\gamma$ is predictive of the convergence of our more practical algorithms.
\begin{theorem}{\textbf{Convergence of ideal proximal basin hopping.}} \label{th:ideal_pbh}
    Suppose that Assumptions \ref{ass:f_properties1}, \ref{ass:f_properties3}, and \ref{ass:Vx_compact_argmin} hold. Let $x_0 \in \RR^d$, $\gamma_0>0$, $\eta>1$, iterate for $k=0,1,\ldots$
    \begin{align}
        x_{k+1} & \in  \argmin_{z \in \mathcal{M}_{x_k}} f(z) \\
        \gamma_{k+1} &= \eta\gamma_k\mathbf{1}_{\left\{f(x_{k+1}) = f(x_k)\right\}} + \gamma_k \mathbf{1}_{\left\{f(x_{k+1}) < f(x_k)\right\}} 
    \end{align}
    Then, $\{x_k\}_{k\in\mathbb{N}}$ converges to $x_\ast$ in finitely many iterations.%
\end{theorem}
\subsection{Proximal basin hopping with exact expectation.} Equipped with the proximal interpretation of $S_f^\gamma$, we introduce a relaxation by replacing the proximal operator by its zeroth-order approximation of parameter $\delta>0$ \citep{osher2023hamilton,lauga2026prox}, playing the role of the temperature, for all $x\in \RR^d$
\begin{equation}
  S_f^\gamma(x)\approx_{\delta} T_f\left(\frac{\mathbb{E}_{y \sim \mathcal{N}(x,\delta \gamma\Id)}[T_f(y) \exp(-f\left(T_f(y) \right)/\delta)]}{\mathbb{E}_{y \sim \mathcal{N}(x,\delta \gamma\Id)}[\exp(-f\left(T_f(y)\right)/\delta)]}\right). \label{eq:exact_expect_pbh}
\end{equation}
With respect to the ideal operator, we want to understand how the approximated proximal operator actually converges to a solution of $\argmin V_x$ when $\delta$ goes to $0$. %
In the small $\delta$ regime, this approximation behaves closely to the ideal $S_f^\gamma$, with the only caveat that if $S_f^\gamma$ is set-valued, the approximation sits near the convex hull of this set.  As before, proofs are deferred to App. \ref{app:proofs_exact_expect_pbh}.

We have this first result, derived from large deviation theory \citep{denhollander2000large} and the properties of the Gaussian distribution. It is the foundation of our analysis.

\begin{lemma} \label{lm:delta_Fm}Suppose that Assumptions \ref{ass:f_properties1}, \ref{ass:Vx_compact_argmin} hold. 
    Let $x\in\mathbb{R}^d$, $\gamma>0$, $\delta>0$, and let $
Y_\delta \sim \mathcal{N}(x,\delta\gamma I_d)$. 
Let $M \in \mathcal M$. %
Then the family $(Y_\delta)_{\delta>0}$ satisfies, for each $M$,
\begin{equation*}
\lim_{\delta\downarrow 0}\delta \log \mathbb{P}(Y_\delta\in \Att(M)) = -\frac{1}{2\gamma}\mathrm{dist}(x,\overline{\Att(M)})^2.
\end{equation*}
\end{lemma}

One can see that weighting this probability by $e^{-f(M)/\delta}$ will make the potential $V_x$ appear on the right-hand side.
Hence, we can show that the ratio of expectations in Eq. \eqref{eq:exact_expect_pbh} converges to the set of minimizers of $V_x$ when $\delta$ goes to $0$. This is easy to obtain if the number of minimizing components is finite (App. \ref{app:proofs_exact_expect_pbh}, Proposition \ref{prop:concentration_finite}).
As we have infinite number of local-minimizer components, the crucial part is to make sure that almost all the mass is in a bounded set where the finite-number-of-components argument will help us complete the proof\footnote{An easy way to make sure that all the mass is in a compact set is to truncate the attained minimizers if they have a norm above a threshold $R$, thus falling back on the finite minimizers case.}. One can see that $f$ cannot be too flat in all directions for this property to hold, otherwise it would allow some of the mass to escape to infinity. Coercivity is not enough to contain the mass, hence we make an additional growth assumption on $f$.
\begin{assumption}\label{ass:growth}
    There exist $C_1,C_2,C_3>0$ and $0<\alpha<\beta$ such that 
    \begin{equation}
        (\forall z \in \RR^d), \qquad C_1\Vert z\Vert^\alpha -C_2\leq f(z) \leq C_3(1+\Vert z \Vert)^\beta. 
    \end{equation}
\end{assumption}
First, we show that a measure induced by the exponential weights goes to $1$ when $\delta \downarrow 0$ on bounded sets of $\mathcal{M}_\ast$.
\begin{proposition} \label{prop:measure_concentration} Suppose that Assumptions \ref{ass:f_properties1}, \ref{ass:f_properties3}, \ref{ass:Tf}, \ref{ass:Vx_compact_argmin}, and \ref{ass:growth} hold.
Let $x\in\RR^d$, $\gamma,\delta>0$, and $Y_\delta\sim\mathcal N(x,\delta\gamma I_d)$. For any Borel $B \in \mathcal{M}_\ast$, we define
\begin{equation}
\nu_{x,\delta}(B):=
\mathbb E \left[e^{-f(T_f(Y_\delta))/\delta}\mathbf 1_{\{T_f(Y_\delta)\in B\}}\right],
\qquad
\mu_{x,\delta}(B):=\frac{\nu_{x,\delta}(B)}{\nu_{x,\delta}(\mathcal{M}_\ast)},
\end{equation}
For any bounded set $K \subset \mathcal{M}_\ast$, $
\mu_{x,\delta}(K)\xrightarrow[\delta\downarrow 0]{} 1.
$
\end{proposition}
In addition, we can also ensure that the mass outside bounded sets does not escape too far (App. \ref{app:proofs_exact_expect_pbh}, Proposition \ref{prop:gaussian_tail_poly}).
Finally, we obtain our convergence result.%
\begin{proposition}\label{prop:barycenter_in_convex_hull} Suppose that Assumptions \ref{ass:f_properties1}, \ref{ass:f_properties3}, \ref{ass:Tf}, \ref{ass:Vx_compact_argmin}, and \ref{ass:growth} hold. Let $x\in \RR^d$, $\delta,\gamma>0$ and $Y_\delta \sim \mathcal{N}(x,\delta \gamma \Id)$.
Denote by $\mathrm{conv}(\mathcal{M}_x)$ the convex hull of $\mathcal{M}_x$, the set of minimizers of $V_x$ (Eq. \ref{eq:fact1}). We have 
\begin{equation}
    \mathrm{dist}\left( \frac{\mathbb{E}[T_f(Y_\delta)e^{-f(T_f(Y_\delta))/\delta}]}{\mathbb{E}[e^{-f(T_f(Y_\delta))/\delta}]},\mathrm{conv}(\mathcal{M}_x)\right) \xrightarrow[\delta \downarrow 0]{} 0
\end{equation}
In particular, if $\mathcal{M}_x = \{m_{x,\ast}\}$, i.e., an isolated local minimizer, then,
\begin{equation}
    \frac{\mathbb{E}[T_f(Y_\delta)e^{-f(T_f(Y_\delta))/\delta}]}{\mathbb{E}[e^{-f(T_f(Y_\delta))/\delta}]} \xrightarrow[\delta \downarrow 0]{} m_{x,\ast}.
\end{equation}
\end{proposition}

This result is almost equivalent to that of the finite case: the weighted barycenter onto which $m_\delta$ converges is inside the convex hull (App. \ref{app:proofs_exact_expect_pbh}, Proposition \ref{prop:concentration_finite}). In general, $m_\delta$ does not need to lie inside the convex hull as some of the mass still lies outside the set of interest as long as $\delta >0$.

The convergence now relies on playing with the $\gamma$ parameter to identify competing minimizers (Remark \ref{remark:2}) and enable concentration on one element. If, in some iterations, several components are tied, it is not really a problem as long as it does not happen indefinitely. And we can guarantee that it does not by imposing a simple rule: if the new guess is of higher function value than the current one, we increase $\gamma$. However, with respect to the first ideal algorithm, we need $\gamma$ to grow comparatively slower than $\delta$, in order to avoid breaking our concentration result, hence we decrease $\delta$ at each iteration (while $\gamma$ does not change if progress is made). Note that we do not reject equivalent samples in terms of function values for the sake of exploration.
 \begin{algorithm}[H]
\caption{Exact expectation proximal basin hopping}
\label{alg:exact_expect_pbh}
\begin{algorithmic}[1]
\Require $x_0 \in \mathbb{R}^d$, $1<\eta_1<\eta_2$,
$\gamma_0>0$, $\delta_0>0$ (small)
\For{$k=0,1\dots,$}
    \State $m_{\delta_k,\gamma_k} \gets \frac{\mathbb{E}[T_f(Y)e^{-f(T_f(Y))/\delta_k}]}{\mathbb{E}[e^{-f(T_f(Y))/\delta_k}]}$, $Y \sim \mathcal{N}(x_k,\delta_k \gamma_k)$
    \State $x_{+} \gets T_f\left( m_{\delta_k,\gamma_k}\right)$
    \If{$f(x_{+})>f(x_k)$}
        \State $x_{k+1} \gets x_{k}$,$\gamma_{k+1} \gets \eta_1 \gamma_k$
    \ElsIf{$f(x_+)=f(x_k)$}
        \State $x_{k+1} \gets x_+$, $\gamma_{k+1} \gets \eta_1 \gamma_k$
    \Else
        \State $x_{k+1} \gets x_+$,$\gamma_{k+1} \gets \gamma_k$
    \EndIf
    \State $\delta_{k+1} \gets \frac{\delta_k}{\eta_2}$
\EndFor
\State \Return $x_{\mathrm{end}}$
\end{algorithmic}
\end{algorithm}

\begin{theorem}{\textbf{Convergence of exact expectation PBH.}} \label{th:exact_expect_pbh} Suppose that Assumptions \ref{ass:f_properties1}, \ref{ass:f_properties3}, \ref{ass:Tf}, \ref{ass:Vx_compact_argmin}, and \ref{ass:growth} hold. Then, the iterates $\{x_k\}_{k\in \mathbb N}$ of Algorithm \ref{alg:exact_expect_pbh} converge to $x_\ast$ when $k\to +\infty$. 
\end{theorem}

\subsection{Proximal basin hopping with approximated expectation.} In practice, we estimate our operator $S_f^\gamma$ with a certain number of samples. It immediately begs the question, can this be sufficient to obtain convergence with reasonably high probability? The answer is yes, and we will show it by studying the case where $N$ actually goes to infinity to understand it. Again, we defer the proofs to App. \ref{app:proofs_approx_expect_pbh}. 
\begin{equation}
  S_f^\gamma(x)\approx_{\delta}^N T_f\left(\frac{\sum_{i=1}^N T_f(y_i) \exp(-f\left(T_f(y_i) \right)/\delta)}{\sum_{i=1}^N \exp(-f\left(T_f(y_i) \right)/\delta)}\right) \text{where }  y_i \sim \mathcal{N}(x,\delta \gamma \Id) .
\end{equation}

\paragraph{Convergence when $N\to+\infty$.}
As for the exact case, we are first concerned with the well-posedness of the operator $m_{N,\delta,\gamma}$ for all $x\in \RR^d$ (line 2 Alg. \ref{alg:approx_expect_Ninfty_pbh}). With standard arguments, we recover that $m_{N,\delta,\gamma} \rightarrow m_{\delta,\gamma}$ when $N\to +\infty$ (App. \ref{app:proofs_approx_expect_pbh}, Lemma \ref{lm:well_posed_N}).
The convergence of $T_f(m_{N,\delta,\gamma})$ to $T_f(m_{\delta,\gamma})$ is however not directly guaranteed as $m_{\delta,\gamma}$ may lie on the boundary of some $\Att(M)$ (where we would lose local continuity of $T_f$), or may lie in the null set, which consists of points from which $T_f$ will not converge to a local minimizer. The first situation is covered by Lemma \ref{lm:m_N_good} and \ref{lm:cluster-points-nearby-components} (App. \ref{app:proofs_approx_expect_pbh}). The null set is handled directly by our algorithm. With our built-in increase of $\gamma$, $m_{\delta,\gamma}$ will fall under the conditions of Lemma \ref{lm:m_N_good}. Indeed, for sufficiently large $\gamma$, and sufficiently small $\delta$, $m_{\delta,\gamma}$ is close to $x_\ast$ (Proposition \ref{prop:barycenter_in_convex_hull}). Hence increasing $N$ sufficiently does the job.
\begin{algorithm}[H]
\caption{Approximated expectation proximal basin hopping}
\label{alg:approx_expect_Ninfty_pbh}
\begin{algorithmic}[1]
\Require $x_0 \in \mathbb{R}^d$, $1<\eta_1<\eta_2$,
$\gamma_0>0$, $\delta_0>0$ (small), $C\geq1$, $N_0 \in \mathbb{N}\setminus\{0\}$,
\For{$k=0,1\dots,$}
    \State $m_{N_k,\delta_k,\gamma_k} \gets \frac{\sum_{i=1}^{N_k}  T_f(y_i) \exp(-f\left(T_f(y_i) \right)/\delta)}{\sum_{i=1}^{N_k}  \exp(-f\left(T_f(y_i) \right)/\delta)}$, $y_1, y_2,\ldots, y_{N_k} \sim \mathcal{N}(x_k,\delta_k \gamma_k \Id)$
    \State $x_{+} \gets T_f\left( m_{{N_k},\delta_k,\gamma_k}\right)$
    \If{$f(x_{+})>f(x_k)$}
        \State $x_{k+1} \gets x_{k}$,  $\gamma_{k+1} \gets \eta_1 \gamma_k$
    \ElsIf{$f(x_+)=f(x_k)$}
        \State $x_{k+1} \gets x_+$, $\gamma_{k+1} \gets \eta_1 \gamma_k$
    \Else
        \State $x_{k+1} \gets x_+$, $\gamma_{k+1} \gets \gamma_k$
    \EndIf
    \State $\delta_{k+1} \gets \frac{\delta_k}{\eta_2}$, $N_{k+1} \gets C \times {N_k} $
\EndFor
\State \Return $x_{\mathrm{end}}$
\end{algorithmic}
\end{algorithm}
\begin{theorem}{\textbf{Convergence of approximate expectation PBH ($N \rightarrow +\infty$).}} \label{th:approx_expect_infty_pbh} Suppose that Assumptions \ref{ass:f_properties1}, \ref{ass:f_properties3}, \ref{ass:Tf}, \ref{ass:Tf_continuity}, \ref{ass:Vx_compact_argmin}, and \ref{ass:growth} hold.
    The iterates $\{x_k\}_{k\in \mathbb N}$ of Algorithm \ref{alg:approx_expect_Ninfty_pbh} (i.e., $C>1$) converge to $x_\ast$ when $k\to +\infty$. 
\end{theorem}
\paragraph{Convergence when $N<+\infty$.} We now address the convergence of our algorithm when the number of samples remains finite. By quantifying that for $N$ large enough $\Vert m_{N,\delta,\gamma}-m_{\delta,\gamma}\Vert$ is small with high probability, we can use the arguments of the proof of Theorem \ref{th:approx_expect_infty_pbh} to recover the convergence.
\begin{proposition} \label{prop:high_proba_m} Suppose that Assumptions \ref{ass:f_properties1}, \ref{ass:f_properties3}, \ref{ass:Tf}, \ref{ass:Tf_continuity}, \ref{ass:Vx_compact_argmin}, and \ref{ass:growth} hold. Let $x \in \RR^d$, $\gamma,\delta>0$.
    For any $\varepsilon>0$, if $N \in \mathbb{N}\setminus\{0\}$ is large enough there exist $\alpha(\varepsilon)>0$, such that for $N$ i.i.d. samples $y_i$ drawn from $Y_\delta \sim \mathcal{N}(x,\delta \gamma \Id)$,
    \begin{equation}
        \mathbb{P}\left(\left\Vert m_{N,\delta,\gamma}-m_{\delta,\gamma}\right\Vert \leq \varepsilon \right)\geq 1 -\frac{\alpha(\varepsilon)}{N}. 
    \end{equation} 
\end{proposition}
In fact, the previous result holds for all $N$, but is vacuous for $N$ too small. It also holds for fixed $\delta$ and $\gamma$ so we need to be careful in bringing $\delta$ to $0$ too fast for theoretical and practical reasons. %
\begin{theorem}
    {\textbf{Convergence of approximate expectation PBH ($N < +\infty$).}} \label{th:approx_expect_finite_pbh} Suppose that Assumptions \ref{ass:f_properties1}, \ref{ass:f_properties3}, \ref{ass:Tf}, \ref{ass:Tf_continuity}, \ref{ass:Vx_compact_argmin}, and \ref{ass:growth} hold.
    Then, the iterates $\{x_k\}_{k\in \mathbb N}$ of Algorithm \ref{alg:approx_expect_Ninfty_pbh} with finite $N$ (i.e., $C= 1$) converge to $x_\ast$ when $k\to +\infty$ with high probability if $N$ is sufficiently large.
\end{theorem}
The required $N$ is path-dependent, but does not need to hold for all the iterates, only for those that have an underlying $m_{\delta,\gamma}$ able to send them to $x_\ast$. Practical considerations are discussed in App. \ref{app:details_alg}, notably the interplay between Proposition \ref{prop:high_proba_m} and the values of $\delta$ and $\gamma$. Remark that the local minimizations of the $y_i$'s can be parallelized.
\subsection{Proximal basin hopping with approximated expectation and local minimization.} Now, if we also assume that we cannot do complete local minimization but only obtain approximate local solutions, then we have a third level of approximations in our algorithm.
\begin{equation}
  S_f^\gamma(x)\approx_{\delta,\varepsilon}^N T_f^\varepsilon \left(\frac{\sum_{i=1}^N T_f^\varepsilon(y_i) \exp\left(-f\left(T_f^\varepsilon(y_i) \right)/\delta\right)}{\sum_{i=1}^N \exp\left(-f\left(T_f^\varepsilon(y_i) \right)/\delta\right)}\right), \, \text{where }  y_i \sim \mathcal{N}(x,\delta \gamma \Id) 
\end{equation}
where for all $x \in \RR^d,$ $f (T_f^\varepsilon (x))\leq f (T_f(x))+ \varepsilon$, and $\Vert T_f^\varepsilon(x) - T_f(x) \Vert \leq r(\varepsilon)$ for some $\varepsilon\geq0$ and function $r:\RR_+ \to \RR_+$, such that $r(0):=0$. If $T_f:=\Id$ we recover the zeroth-order proximal algorithm, which converges provided that $N\to+\infty$.
A first observation is that to obtain convergence we will need $\frac{\varepsilon}{\delta}\rightarrow 0$ with the iterations as:
$
\exp(-f(T_f^\varepsilon(x))/\delta) \geq \exp\left(-f(T_f(x))/\delta\right) \exp\left(-\frac{\varepsilon}{\delta}\right)$. The right-hand side exponential is approximately the size of the perturbation on the weights. A second observation is that we also need $\Vert T_f^\varepsilon(x) - T_f(x) \Vert$ to go to $0$ uniformly for all $x$ w.r.t. $\varepsilon$; otherwise the weighted average would not behave as $m_{N,\delta,\gamma}$, let alone for the refined guess.

\section{Numerical experiments}
In this section, we benchmark our algorithm against the zeroth-order proximal (ZOP) algorithm presented in \citep{zhang2024inexact}, and basin hopping (BH)  algorithm \citep{wales1997basinhopping}. We argue that these two algorithms are sufficient as benchmarks as other global optimization algorithms lack convergence guarantees, and because ZOP was shown to often be better than them in \citep{zhang2024inexact} as was also shown for BH \citep{baioletti2024performance}.

\paragraph{Comparison setup.} Basin hopping is our Algorithm \ref{alg:approx_expect_Ninfty_pbh} with sample size of $1$. The new guess, however, is accepted with probability proportional to $\exp(\frac{f(x)-f\left(T_f(y_i) \right)}{\delta})$. Zeroth-order Prox is our Algorithm App. Exp. PBH with $T_f = \Id$. As the three algorithms are not comparable w.r.t. to iteration count, we choose to report CPU time instead. We set parameters of all algorithms to the same value at initialization, and update with the rules of our algorithm, keeping the number of samples constant for PBH and ZOP. We did not investigate multisearch, and instead focused on convergence from one initialization. 
After the CPU-time budget has been spent, algorithms are stopped. 

\paragraph{Set of problems.} We propose to test our algorithm on various problems, with different complexities.

\textit{Hard synthetic functions.}
We test our algorithm against the $d$-dimensional Rastrigin, and $d$-dimensional Griewank functions, both satisfying our growth assumption. The global minimizer is at $\mathbf{0}$ for all $d$, with $f(\mathbf{0})=0$. Rastrigin displays a much higher number of local minimizers than the Griewank function. 
We optimized these two functions with dimensions ranging from $2$ to $200$. The local solver uses gradient descent. %

\textit{Lennard--Jones cluster.}
Let $x_1,\dots,x_D \in \mathbb{R}^3$ denote the positions of $D\geq2$ atoms.
The Lennard--Jones energy is defined by
\[
\mathrm{LJ}(x_1,\dots,x_D)
:=
\sum_{1\leq i<j\leq D}
\left(
\frac{1}{\|x_i-x_j\|^{12}}-\frac{2}{\|x_i-x_j\|^{6}}
\right).
\]
The associated global optimization problem is of dimension $d:=3 D$. The energy depends only on pairwise distances, and is therefore invariant under translations, rotations, and permutations. We optimized this function for $3$ to $50$ atoms. The local solver uses LBFGS steps.

\textit{Fitting scaling laws.} Fitting scaling laws is an important problem in deep learning to predict model performance as a function of model size $K$ and the number of tokens $D$. In particular, the choice of dataset mixtures is of crucial importance. 
To derive a good scaling law, the authors of \citep{shukorscaling} collected training runs with different domain weights $h$ (datasets), model sizes $K$, and token budgets $D$, and recorded the resulting loss on the target domain. Three scaling law models were tested in \citep{shukorscaling} to understand optimal data mixtures \citep[Section 2.2]{shukorscaling}. %
The optimal parameters are estimated by minimizing a Huber objective over $p$ observations.
Finally, evaluation is performed on held-out runs with unseen triples $(K,D,h)$ by comparing predicted and observed losses. The local solver uses LBFGS steps.
The code for computing, fitting, and evaluating the scaling laws is reused from \url{https://github.com/apple/ml-scalefit}.

\paragraph{Numerical results.}
All the results displayed in this section are computed using $8$ CPUs. We report final values, and not best achieved values, to highlight convergence capabilities and robustness. For the synthetic functions, we can see that the performance gap increases between PBH and BH when the dimension increases on the Rastrigin function, and stays consistent on the simpler Griewank function. ZOP did not converge on either functions: our hypothesis is that the number of samples needs to increase as $\delta$ decreases \citep[Section 5.2]{zhang2024inexact} in a much more constrained way than PBH. In our experiments, we saw that what matters at initialization is the scaling $\delta \gamma$: reducing $\delta$ at initialization would target smaller function values, but at the cost of exploration, thus $\gamma$ should be increased accordingly.
\begin{figure}
    \centering
    \includegraphics[trim={0em 0em 0em 0em},clip,width=0.32\textwidth]{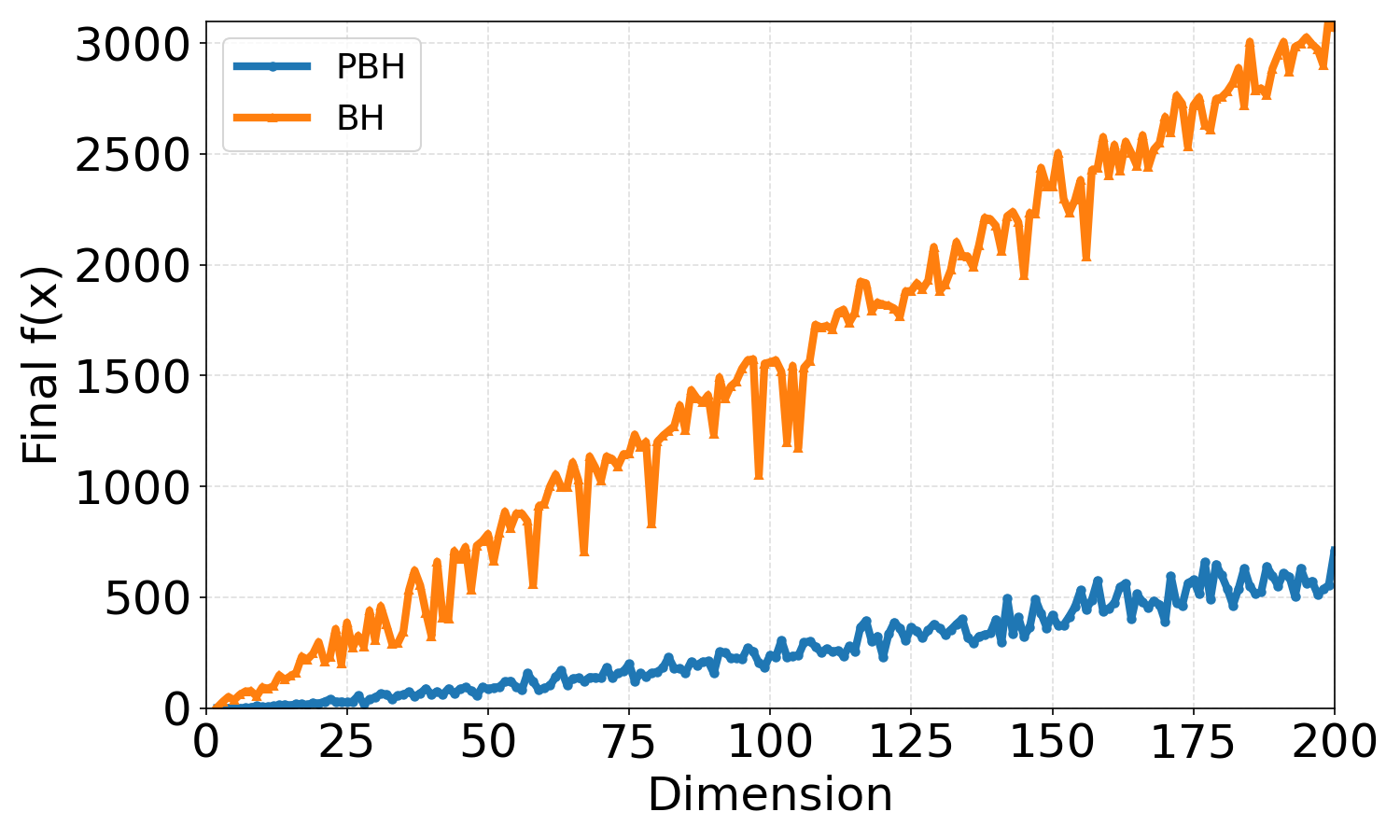}
    \includegraphics[trim={0em 0em 0em 0em},clip,width=0.32\textwidth]{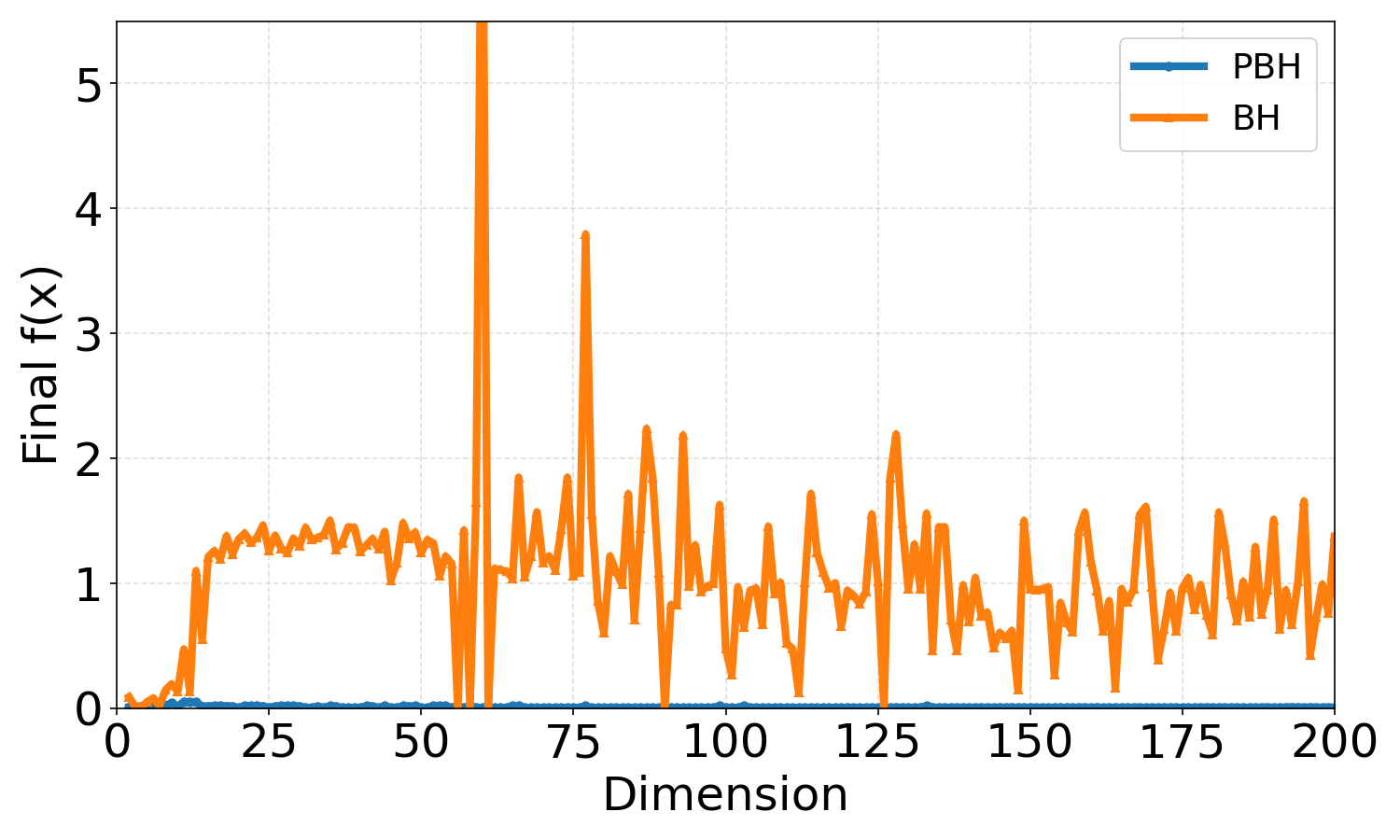}
    \includegraphics[trim={0em 0em 0em 0em},clip,width=0.32\textwidth]{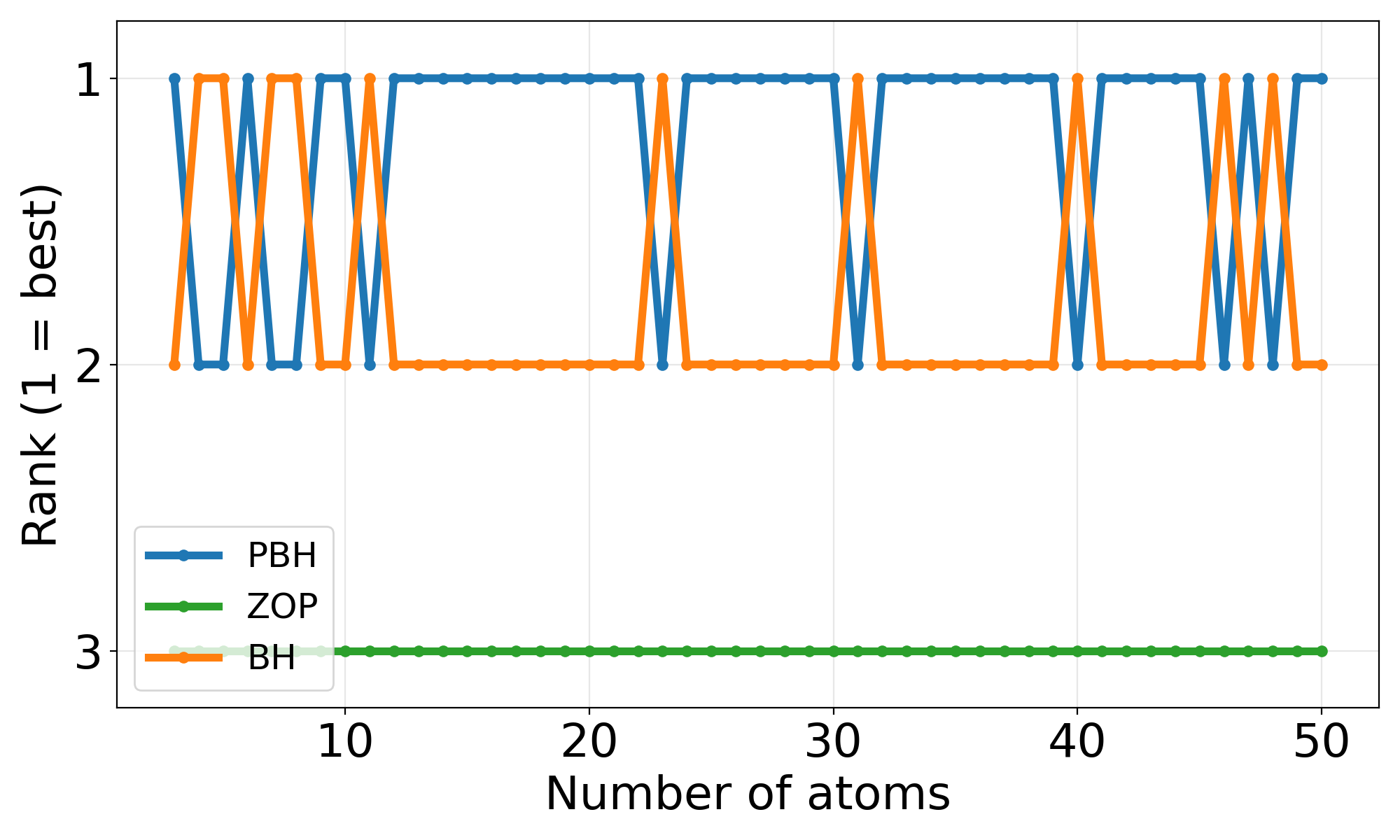} \
    \caption{Comparison of BH (\textcolor{orange}{orange}), PBH (\textcolor{blue}{blue}), and ZOP (\textcolor{ForestGreen}{green}) on (left) Rastrigin, (middle) Griewank, and (right) Lennard--Jones: the ranking between algorithms is plotted ($1$ is best). $N=10\times d$, $10$ local steps, CPU time $=2 \times d$. When it does not appear, ZOP diverged. We set $\gamma = 5$ and $\delta=0.5$ at initialization. Rastrigin experiment ran in $18$h, with $8$ CPUs and $2$GB of memory.}
\end{figure}
In this constrained budget setting, it appears that allocating more time to the exploration is not always useful: it is better to do more iterations with fewer samples. We display the results for the additive scaling law, for the three large deep learning models tested in \citep{shukorscaling}: large language model (LLM), native multi-modal model (NMM), and large vision model (LVM). We report the results for the additive scaling law, that yielded the best test error. For almost all problems PBH with $N=5$ is better than ZOP or BH (Table \ref{tab:add_pbh_mre_by_samples}). Increasing the number of LBFGS steps did not yield better results (for PBH and BH), as increasing the number of samples (keeping CPU number fixed).
\begin{table}[ht]
\centering
\small
\newcommand{\std}[1]{{\scriptsize $\pm$ #1}}
\setlength{\tabcolsep}{2.2pt}
\begin{tabular}{|ll|ccc|cc|}
\hline
Family & Domain & PBH (N=5) & PBH (N=10) & PBH (N=15) & ZOP (N=5) & BH \\
\hline
\multirow{7}{*}{\shortstack{LLM\\($d=19$)}}
& arxiv         & 7.244 \std{4.369}  & \textbf{6.848} \std{4.795}  & 7.336 \std{5.859}  & 7.431 \std{7.533}  & 10.037 \std{5.472} \\
& book          & \textbf{3.671} \std{2.804}  & 4.136 \std{2.778}  & 3.752 \std{3.382}  & 6.158 \std{6.523}  & 4.915 \std{3.298} \\
& c4            & \textbf{4.105} \std{3.225}  & 4.528 \std{3.622}  & 4.168 \std{3.460}  & 7.061 \std{6.521}  & 6.977 \std{4.146} \\
& commoncrawl   & \textbf{3.964} \std{2.848}  & 5.172 \std{3.235}  & 4.009 \std{3.615}  & 5.638 \std{6.520}  & 6.843 \std{3.820} \\
& github        & \textbf{8.597} \std{7.418}  & 9.899 \std{7.707}  & 8.961 \std{8.241}  & 11.131 \std{9.351} & 13.191 \std{8.458} \\
& stackexchange & \textbf{6.109} \std{3.795}  & 7.024 \std{4.270}  & 6.374 \std{5.469}  & 7.651 \std{6.644}  & 8.640 \std{6.598} \\
& wikipedia     & \textbf{8.540} \std{5.088}  & 9.484 \std{5.011}  & 8.554 \std{5.340}  & 10.416 \std{6.524} & 13.898 \std{5.064} \\
\hline
\multirow{4}{*}{\shortstack{LVM\\($d=13$)}}
& alttext       & \textbf{3.083} \std{3.745}  & 6.037 \std{3.585}  & 4.896 \std{3.674}  & 6.354 \std{4.648}  & 6.452 \std{4.459} \\
& highquality1  & 5.753 \std{3.368}  & 7.636 \std{3.903}  & 6.474 \std{3.091}  & 14.093 \std{6.332} & \textbf{5.494} \std{3.799} \\
& highquality2  & \textbf{3.978} \std{2.948}  & 8.189 \std{7.543}  & 4.681 \std{7.894}  & 14.201 \std{12.136} & 5.446 \std{4.691} \\
& synthetic     & 11.384 \std{12.289} & 16.939 \std{13.689} & \textbf{11.302} \std{11.621} & 24.611 \std{15.638} & 14.320 \std{7.138} \\
\hline
\multirow{3}{*}{\shortstack{NMM\\($d=11$)}}
& captions      & \textbf{3.725} \std{3.209}  & 4.526 \std{3.400}  & 4.074 \std{3.891}  & 8.698 \std{7.172}  & 5.741 \std{4.222} \\
& interleaved   & 2.854 \std{1.329}  & 3.299 \std{1.318}  & \textbf{2.693} \std{1.470}  & 6.176 \std{2.263}  & 5.618 \std{2.622} \\
& text          & 2.387 \std{1.185}  & 3.169 \std{1.320}  & \textbf{2.341} \std{1.313}  & 4.500 \std{1.894}  & 7.332 \std{3.689} \\
\hline
\end{tabular}
\caption{Test \(
    \mathrm{MRE}
    =
    \frac{|\text{prediction} - \text{observation}|}{\text{observation}}\) (\%) for fitting additive scaling law. 
(\texttt{lbfgs steps=5}, \texttt{CPU time=300s}).
We set $\gamma = 10$ and $\delta = 10^{-3}$ at initialization. Results averaged over three seeds, we report the standard deviation. On average, PBH yields better test error than other algorithms. The standard deviation is also indicative of the best possible value reachable.
} 
\label{tab:add_pbh_mre_by_samples}
\end{table}

\section{Conclusion, limitations and broader impact}
\paragraph{Conclusion.} In this paper, we present a new theoretical framework, Proximal Basin Hopping, from which we construct a practical algorithm outperforming state-of-the-art methods on various problems. 

\paragraph{Limitations.} With respect to the present convergence analysis, and the subsequent algorithm, a question can be raised about the fact that we aggregate points, instead of simply taking the best point every time. In this setting, we would lose the proximal interpretation which is vital for convergence to the global minimizer through the parameter $\gamma$, but the practical results may be better. 
It has been shown for the basin hopping algorithm that in some cases, better-tailored perturbation could improve optimization \citep{englander2014tuning}. Here, if we steer away from a Gaussian distribution for the stochastic exploration, we also steer away from the proximal interpretation and the convergence analysis of our algorithm. However, a variable metric \citep{chouzenoux2014variable} is possible by using $y \sim \mathcal{N}(x_k, \delta \gamma \Sigma) \text{ instead of } y \sim \mathcal{N}(x_k, \delta \gamma \Id)$. It remains to choose $\Sigma$ properly.  
Also, the impact of the inexactness of the local minimization should be investigated on a per-solver basis. A first step should be to investigate it for gradient flow. Significant work is required as the geometry of attraction basins should be described in far more detail than what we used here. 
Finally, if letting the number of samples be fixed simplifies greatly the computations, it could be interesting to adapt the number of samples on the fly (notably with Proposition \ref{prop:high_proba_m}), to temporarily escape some traps, or reduce the computational load if progress continues.

\paragraph{Broader impact.} The proposed method may help reduce the number of searches required to  obtain global minimizer, and reduce the overall computational load.

\medskip

\bibliographystyle{plainnat}
\bibliography{references}

\newpage
\appendix
\section{Related works}

\paragraph{Global optimization and zeroth-order methods.}
Global optimization has a long history, with most used methods being derived from some metaheuristics \citep{locatelli2021global}. A large part of this literature can be viewed as zeroth-order optimization, as the algorithms only query function values. The simplest one being pure random search, which samples candidate points and keeps the best value found \citep{zabinsky2009random}. More structured deterministic zeroth-order methods include Lipschitz global optimization, but are typically limited by dimension \citep{hansen1992global,bouttier2020regret}. Multi-start local search is another way to combine global exploration and local descent, and its efficiency depends obviously on how starting points are allocated \citep{gyorgy2011efficient}.

Genetic and evolutionary algorithms maintain and transform a population of candidate solutions through selection and variation mechanisms \citep{lambora2019genetic}. Simulated annealing instead constructs a stochastic process whose ability to escape local minima is controlled by a temperature parameter; classical convergence guarantees require sufficiently slow cooling schedules \citep{hajek1988cooling} if the number of states to investigate is finite. Particle swarm optimization updates a population of particles using both individual and collective best positions \citep{kennedy1995particle}. Consensus-based optimization is another population-based approach in which particles concentrate around weighted averages favoring low objective values; it is a more theoretically principled way of doing particles optimization, as global convergence under suitable assumptions was derived \citep{fornasier2024consensus}. Bayesian optimization and related surrogate-based methods are highly effective in low to moderate dimension, but their behavior in high dimension remains delicate \citep{papenmeier2025understanding}. 

\paragraph{Energy landscapes, basins, and local optima networks.}
A complementary line of work, that has attracted most of its attention from chemistry fields, studies the organization of nonconvex landscapes through their local minima and basins of attraction. Local optima networks encode local minimizers as nodes and transitions between them as edges, giving a graph representation of the landscape \citep{ochoa2014local,tomassini2022local}. Such tools have been used to visualize and quantify funneling landscapes, in which local descent tends to guide configurations toward increasingly low-energy regions \citep{wales2010energy}. The structure of attraction basins evidently depends on the local solver: different minimization algorithms may induce different basins on the same objective \citep{asenjo2013visualizing}. This viewpoint is closely related to the study of iteration maps, attractors, and basins of attraction in numerical minimization \citep{levy2018attraction,levy2026analyzing}. 

\paragraph{Basin-hopping algorithms.}
Basin hopping was introduced for the global optimization of atomic cluster energies, in particular Lennard--Jones clusters \citep{wales1997basinhopping}. The central idea is to compose random perturbations with local minimization, thereby transforming the original objective into an energy landscape over local minima. This transformation has proved especially effective in molecular and cluster optimization \citep{prentiss2008protein,alvarez2025review,de2025goat}. Despite its practical success, this scheme has limited general convergence theory and is typically analyzed or tuned empirically.

Several variants of basin hopping have since been created, due to its success. Monotonic basin hopping accepts only moves that improve the objective and has been tuned for trajectory optimization problems \citep{englander2014tuning}. Adaptive hopping strategies modify the perturbation distribution to improve transitions between promising regions \citep{englander2020hopping}. Population-based basin hopping maintains several candidate solutions and uses dissimilarity or diversity mechanisms to explore multiple regions of the landscape \citep{grosso2007population}. Recent comparisons show that basin hopping is competitive with established metaheuristics on global optimization benchmarks, but its performance remains problem dependent \citep{baioletti2024performance}. There is also theoretical work studying the difficulty of moving between distant basins \citep{goodridge2022hopping}. 

\paragraph{Proximal algorithms and stochastic proximal approximations.}
Proximal algorithms are a central tool in convex and variational optimization \citep{combettes2005signal,parikh2014proximal,VarAnalRockafellar}. In convex settings, inexact proximal splitting methods remain convergent when the errors in the proximal evaluations are controlled, for instance by summability assumptions \citep{combettes2005signal}. Recent work also studies Monte Carlo approximations of proximal steps for nonsmooth convex optimization \citep{di2025monte}.

A Hamilton--Jacobi-based zeroth-order approximation of the proximal operator was proposed in \citep{osher2023hamilton}:
\begin{equation}
    \prox_{\lambda f}^{\delta}(x)
    =
    \frac{\mathbb{E}_{y \sim \mathcal{N}(x,\delta \lambda I)}
    \left[y \exp(-f(y)/\delta)\right]}
    {\mathbb{E}_{y \sim \mathcal{N}(x,\delta \lambda I)}
    \left[\exp(-f(y)/\delta)\right]},
    \qquad
    \prox_{\lambda f}(x)=\lim_{\delta\downarrow 0}\prox_{\lambda f}^{\delta}(x).
\end{equation}
Rates for such stochastic proximal and projection estimators have been investigated in \citep{lauga2026prox,morales2026convergence}. This approximation has also been used for zeroth-order global optimization in \citep{zhang2024inexact}, where an adaptive proximal-point scheme encourages exploration near non-global critical points. Their implementation estimates the expectation either by tensor-train methods or by Monte Carlo integration. The convergence argument relies on decreasing the smoothing parameter and increasing the number of samples sufficiently fast so that the accumulated proximal error remains controlled, for example through summability conditions \citep[Corollary 4]{zhang2024inexact}. 

%

\section{Details about our algorithm} \label{app:details_alg}
\paragraph{In practice.}
Instead of decreasing $\delta$ at each iteration, we adjust it on the fly: if the minimum value of $f$ identified by the sampling mechanism has not lead to concentration on this minimum value, we decrease $\delta$ accordingly. This small tweak is a practical tradeoff given that we are using small sample sizes.

Furthermore, for numerical stability, the weighted mean is computed by substracting the maximum weight value, before computing the ratio. It is also done for the ZOP algorithm (for the same reasons).
\begin{proposition} \label{prop:sf_min}
    Suppose that Assumptions \ref{ass:f_properties1}, \ref{ass:f_properties3}, and \ref{ass:Tf} hold. Let $N\in \mathbb{N}\setminus\{0\}$, let $\{m_i\}_{1\leq i \leq N} \in \mathcal{M}_\ast$, and assume that there exists $1\leq i_x \leq N$ such that $f(m_{i_x})<f(m_i)$ for all $i\neq i_x$, with $m_{i_x} \in \mathbb B(m_{i,x},r_{i,x}) \subseteq \Att(m_{i_x})^\circ$. If $\delta$ is small enough, i.e., such that
    \begin{equation}
        e^{-\Delta/\delta} \leq \frac{r_{i_x}}{(N-1)\max_{1\leq i \leq N} \Vert m_i - m_{i_x} \Vert},
     \end{equation}
    where $\Delta = \min_{1\leq i \leq N, i \neq i_x} f(m_i) -f(m_{i_x})$, 
    then
    \begin{equation}
        T_f\left(\frac{\sum_{i=1}^N m_i \exp(-f\left(m_i \right)/\delta)}{\sum_{i=1}^N \exp(-f\left(m_i \right)/\delta)}\right) = m_{i_x}.
    \end{equation}
\end{proposition}
\begin{proof}
    We note 
    \[m_\delta:=\frac{\sum_{i=1}^N m_i \exp(-f\left(m_i \right)/\delta)}{\sum_{i=1}^N \exp(-f\left(m_i \right)/\delta)} \text{ and } \lambda_i = \frac{\exp(-f\left(m_i \right)/\delta)}{\sum_{i=1}^N \exp(-f\left(m_i \right)/\delta)} \]
     We need to show that $m_\delta \in \Att(m_{i_x})$ to conclude. $m_{i_x}$ being in the interior of $\Att(m_{i_x})$ there exist $r_{i_x}>0$ such that $\mathbb{B}(m_{i_x},r_{i_x}) \subset \Att(m_{i_x})$. By taking this radius as large as possible, we upper bound the value of $\delta$ required. We want:
     \begin{align*}
        \left\Vert \sum_{i=1}^N \lambda_i m_i - m_{i_x} \right\Vert  \leq r_{i_x}\Leftrightarrow  \left\Vert \sum_{i=1, i\neq i_x }^N \lambda_i (m_i - m_{i_x}) \right\Vert  \leq r_{i_x}
     \end{align*}
     Now,
     \begin{equation*}
        \left\Vert \sum_{i=1, i\neq i_x }^N \lambda_i (m_i - m_{i_x}) \right\Vert \leq  \sum_{i=1, i\neq i_x }^N \lambda_i \Vert m_i - m_{i_x} \Vert \leq (1-\lambda_{i_x}) \max_{1\leq i \leq N} \Vert m_i - m_{i_x} \Vert.
     \end{equation*}
    $\Delta$ is stricly greater than $0$ by assumption. We have
     \begin{equation*}
        \lambda_{i_x} = \frac{1}{1+\sum_{i\neq i_x} \exp(-\left(f\left(m_i \right)-f(m_{i_x})\right)/\delta)},
     \end{equation*} 
     hence,
     \begin{equation*}
        1 - \lambda_{i_x} \leq \sum_{i\neq i_x} \exp(-\left(f\left(m_i \right)-f(m_{i_x})\right)/\delta) \leq (N-1)\exp(-\Delta/\delta).
     \end{equation*}
     Thus for concentration to occur, we need $\delta$ small enough so that 
     \begin{equation*}
        e^{-\Delta/\delta} \leq \frac{r_{i_x}}{(N-1)\max_{1\leq i \leq N} \Vert m_i - m_{i_x} \Vert}.
     \end{equation*}
\end{proof}
This threshold on $\delta$ is a worst case, a multiplicity of $m_{i_x}$ reduces the mass on "bad" local minimizers while increasing the mass on the correct local minimizer.
An other aspect that is not directly covered by Proposition \ref{prop:high_proba_m} is the dependence on the dimension for $\alpha(\varepsilon)$. We can at least quantify the probability of eating the basin of attraction for one sample. It is reasonable to assume that we can fit a ball of non zero radius in each basin of attraction and we can estimate the probability of one sample being in this ball explicitly, and subsequently of at least one among $N$. We plot in Figure \ref{fig:proba_closeness_xstar} an estimation of these probabilities with respect to the dimension and the number of samples.
\begin{lemma} \label{lm:proba_basin}Let $x\in \RR^d$, $\sigma>0$ and $Y_\sigma \sim \mathcal{N}(x,\sigma^2 \Id)$. Let $r>0$. We have
    \begin{equation}
        \mathbb{P}(Y_\sigma \in \mathbb{B}(0,r)) = \mathrm{CDF}_{\mathcal{X}'_d(\Vert x \Vert^2/\sigma^2)}(\frac{r^2}{\sigma^2}).
    \end{equation}
Moreover to have a probability $0<\alpha < 1$ of one sample hitting this ball, we need
\begin{equation}
    N \geq \frac{\log(1-\alpha)}{\log(1-\mathbb{P}(Y_\sigma \in \mathbb{B}(0,r)))}.
\end{equation}
\end{lemma}
\begin{proof}
    We have
\begin{align*}
    \mathbb{P}(Y_\sigma \in \mathbb{B}(0,r))  = \mathbb{P}(\Vert Y_\sigma \Vert \leq r) = \mathbb{P}(\Vert Y_\sigma \Vert^2 \leq r^2) = \mathbb{P}(\Vert Y_\sigma \Vert^2/\sigma^2 \leq r^2/\sigma^2)
\end{align*}
Now, $\Vert Y_\sigma \Vert^2/\sigma^2$ is the sum of $d$ independent standard normal distributions, with non zero means. Hence, it is a non central chi-square distribution:
\begin{equation}
    \frac{\Vert Y_\sigma \Vert^2}{\sigma^2} \sim \mathcal{X}'_d(\lambda), \text{ where } \lambda = \sum_{i=1}^{d} \frac{\mu_i^2}{\sigma^2}.
\end{equation}
Therefore, 
\begin{align*}
    \mathbb{P}(Y_\sigma \in \mathbb{B}(0,r))  = \mathrm{CDF}_{\mathcal{X}'_d(\lambda)}(\frac{r^2}{\sigma^2}).
\end{align*}
Setting $\alpha$ as the probability of one sample to be in this ball when drawing $(y_i)_{1\leq i \leq N}$ samples of $Y_\sigma$, it is defined as:
\begin{equation*}
    \mathbb{P}(\exists ~1 \leq i \leq N, y_i \in \mathbb{B}(0,r) ) = 1-(1-\mathbb{P}(Y_\sigma \in \mathbb{B}(0,r)))^N = \alpha,
\end{equation*}
which yields after a few rewriting:
\begin{equation}
    N \geq \frac{\log(1-\alpha)}{\log(1-\mathbb{P}(Y_\sigma \in \mathbb{B}(0,r)))}.
\end{equation}
\end{proof}
For instance, if $\mathbb{P}(Y_\sigma \in \mathbb{B}(0,r))$ is small and $\alpha =0.95$ we obtain that $N$ needs to be of the order
\begin{equation}
    N \gtrsim \frac{3}{\mathbb{P}(Y_\sigma \in \mathbb{B}(0,r))}.
\end{equation}
\begin{figure}
    \begin{centering}
    \includegraphics[width=0.45\textwidth]{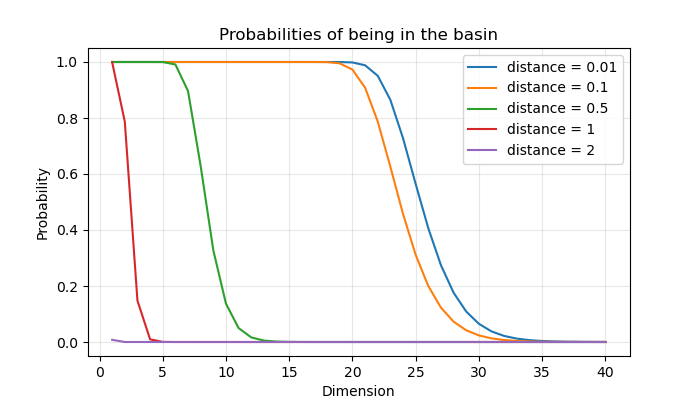} \hspace{1em} \includegraphics[width=0.45\textwidth]{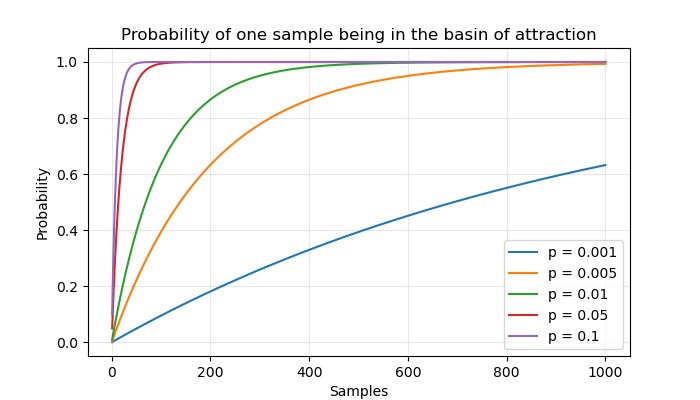}
    \end{centering}
    \caption{Probability of being in a ball in the basin of $x_\ast$ and $r^2/\sigma^2 = 1$. \label{fig:proba_closeness_xstar}}
\end{figure}

Therefore we can guarantee with high probability that we can identify basin of attractions of better local minimizers (from a local minimizer), provided that the number of samples is high enough.  
One should also be wary of the value of $\sigma$ not to spread our samples to much. In order to quantify that more precisely we can use the fact that the cumulative distribution function of the non-central chi-squared distribution can be expressed with the generalized Marcum $Q$-function \citep{sun2010monotonicity}.
\begin{lemma} Let $x\in \RR^d\setminus\{0\}$, $\sigma>0$ and $Y_\sigma \sim \mathcal{N}(x,\sigma^2 \Id)$. Let $r>0$. The probability 
        $\mathbb{P}(Y_\sigma \in \mathbb{B}(0,r))$
    is, with the other parameters fixed, strictly decreasing in the dimension $d$ and the norm $\Vert x \Vert$, and stricly increasing in the ratio $r/\sigma$. On the other hand, with other parameters fixed and $\Vert x \Vert>r$,
    \begin{equation}
        \lim_{\sigma \to 0} \mathbb{P}(Y_\sigma \in \mathbb{B}(0,r)) = 0 \text{ and } \lim_{\sigma \to +\infty} \mathbb{P}(Y_\sigma \in \mathbb{B}(0,r)) = 0,
    \end{equation}
    implying that there exists $\sigma_\ast$ such that 
    \begin{equation}
        (\forall \sigma >0), \, \mathbb{P}(Y_{\sigma_\ast} \in \mathbb{B}(0,r)) \geq \mathbb{P}(Y_\sigma \in \mathbb{B}(0,r))
    \end{equation}
\end{lemma}
\begin{proof}
    We have 
    \[
    \mathrm{CDF}_{\mathcal{X}'_d(\Vert x \Vert^2/\sigma^2)}(\frac{r^2}{\sigma^2}) = 1-Q_{\frac{d}{2}}\left(\frac{\Vert x \Vert}{\sigma}, \frac{r}{\sigma}\right).
    \]
    The first result follows from \citep[Theorem 1(a)]{sun2010monotonicity} as all parameters are strictly positive. The second one is immediate from the fact that $\Vert x \Vert >r$.
\end{proof}

\section{Preliminary results and discussion of the assumptions}
\subsection{Assumptions on $f$} \label{app:ass_f}
\paragraph{Finitely many minimizing components in any compact set.} As stated, $f$ is tame, meaning for any compact set $K$
\begin{enumerate}
    \item $K$ is definable in some o-minimal structure over $\RR$, and $\left.f\right|_{K} : K \to \RR$ is definable: the set of local minima is a definable subset of $\RR^d$ \citep[Lemma 1.7]{fernando2020set} and definable subset have finitely many connected components \citep[Section 3.2]{coste1999introduction}.
    \item $K$ is semi-algebraic, and $\left.f\right|_{K}$ is semi-algebraic: semi-algebraic function are definable \citep{fernando2020set}.
    \item $K$ is globally subanalytic, and $\left.f\right|_{K}$ is globally subanalytic: same as above \citep{fernando2020set} 
    \item $\left.f\right|_{K}$ is a Morse function, i.e., it is $C^2$ with only isolated nondegenerate critical points \citep{audin2014morse}.
    \item $\left.f\right|_{K}$ is a Morse-Bott function, i.e., it is $C^2$, and each connected component of minima is a critical submanifold \citep{rebjock2025fast}. This is the case of the Rastrigin function.
\end{enumerate}
\paragraph{Graph of local minimizers.} Following the basin hopping literature \citep{hajek1988cooling,wales2010energy,tomassini2022local}, and some other recent works \citep{azizian2025global}, in practice we somehow assume that local minimizers of $f$ can be represented with a graph. This assumption is however not required to prove convergence. Here, the edges depend on the local solver $T_f$ in the following sense: for any local minimizer $P\in \mathcal{M}$, the distance from $p$ to the attractor of a better local minimizer $Q$ (in the sense that $f(P)>f(Q)$) is uniformly bounded for all $P$.
\begin{assumption} \label{ass:graph}
    For all $P \in \mathcal{M}\setminus \{x_\ast\}$, there exists $Q \in \mathcal{M}$ such that 
    \begin{equation}
        \sup_{x \in P} \mathrm{dist}(x,\overline{\Att(Q)}) \leq d_f <+\infty \text{ and } f(P)> f(Q)%
    \end{equation}
\end{assumption}
Hence, we construct our graph with vertices as local minimizers (or connected components of local minimizers) and a directed edge between two vertices $P$ and $Q$ if $\mathrm{dist}(P,\overline{\Att(Q)}) \leq d_f$. The edge is undirected if also $\mathrm{dist}(Q,\overline{\Att(P)}) \leq d_f$. From this assumption, the global minimizer $x_\ast$ is reachable by "simply" jumping from one local minimizer to a better local minimizer. This assumption is implicitly made in basin hopping optimization so that random perturbations can reach the basin of a better local minimizer \citep{goodridge2022hopping}. %

We can numerically compute an upper bound of $d_f$ (Assumption \ref{ass:graph}) for the $1-D$ Rastrigin function only has isolated local minimizers, and isolated local maximizers:
\begin{equation}
    \mathrm{Ras}(x) = x^2 + 10 ( 1 - \cos (2 \pi x) ), 
\end{equation}
which are solutions of 
\[
\frac{x}{10\pi} = -\sin(2\pi x).
\]
These solutions can be computed numerically and are approximately away from each other by a little less than $0.5$, hence $d_f =\frac{1}{2}$ is a valid upper bound in this case.
\subsection{Properties of the solver} \label{app:ass_Tf} The definition of the solver in the main body of the paper contains only the relevant assumption to derive our analysis, we lay here a more detailed presentation of what these assumptions entail. Assume that $f$ is $C^2$-smooth (with Lipschitz continuous gradient). Let us define $\varphi:[0,+\infty),\RR^d\mapsto\RR^d$ the gradient flow mapping any $x\in \RR^d$ to a local minimizer of $f$:
\begin{equation}
    \left\{\begin{array}{ll}
        \varphi(0,x) & = x \\
        \varphi(t,x) & = -\nabla f( \varphi(t,x))
    \end{array} \right.
\end{equation} 
Define the critical components of $f$ as:
\begin{equation}
    \mathrm{crit }f := \left\{x \in \RR^d, 0=\nabla f(x)\right\}
\end{equation}
and we assume it is composed of infinitely many connected components.
Following this definition, we have for any component $C\in \mathrm{crit }f $ that
\begin{equation}
    \Att(C) := \left\{x \in \RR^d, \lim_{t\to+\infty} \mathrm{dist}(\varphi(t,x),C)=0\right\}.
\end{equation}
\citep[Lemma D.28]{azizian2024long} holds in our context and thus
for any $x\in \RR^d$, there exists $C \in \mathrm{crit }f$ such that 
\begin{equation}
    \lim_{t\to+\infty} d(\varphi(t,x),C)=0.
\end{equation}
The proof is immediate from \citep[Lemma D.28]{azizian2024long}, as $f$ is coercive and the components $C$ are connected. The finiteness of the number of components is not important, their connectedness is.
$\mathcal{M}$ is the set of asymptotically stable critical components \citep[Lemma D.29]{azizian2024long} (again the proof does not rely on finiteness of critical components), implying  that for all $M \in \mathcal{M}$, $\Att(M)$ is open. Under standard assumptions \citep{lee2016gradient} any $C \in \mathrm{crit }f\setminus \mathcal{M}$, $\Att(C)$ is of measure $0$.
\subsection{Some preliminary results related to the local minimizers} We start this presentation by listing some important properties of the local minimizers and their attractor sets.  
\begin{lemma} \label{lm:null_set}
    Suppose that Assumptions \ref{ass:f_properties1}, \ref{ass:Tf}, and \ref{ass:Vx_compact_argmin} hold. The sets $\left\{\Att(M)\right\}_{M \in \mathcal{M}}$ form a measurable partition of $\RR^d$ up to a null set.
\end{lemma}
\begin{proof}
    First, the sets of attractors of critical points who are non local minimizers is of measure $0$. Let us denote by $N$ this set of points which are not attracted by a local minimizer through $T_f$. We have for all $y\in \RR^d\setminus N$, $T_f(y) \in \mathcal{M}_\ast$ and more precisely, there exists a unique $M \in \mathcal{M}$ such that $T_f(y) \in M$ as the $M$ are pairwise disjoints. Hence, for any $P,Q\in\mathcal{M}$, $\Att(P)\cap \Att(Q) = \emptyset$, and 
    \begin{equation}
        \bigcup_{j}\Att(M_j) = T_f^{-1}(\mathcal{M}_\ast) = \RR^d\setminus N.
    \end{equation} 
    As $T_f$ is measurable and every $M$ is Borel by assumption, we get the desired result. 
\end{proof}
\begin{lemma} \label{lm:Vx}
Let $x\in\mathbb R^d$. For every $\eta>0$, there exists a finite subset $\mathcal K_\eta\subset\mathcal M$
such that
\begin{equation}
\inf_{M\in \mathcal M\setminus \mathcal K_\eta} V_x(M)\geq
\inf_{\widetilde M\in\mathcal M} V_x(\widetilde M)+\eta.
\end{equation}
\end{lemma}
\begin{proof}
Since $\mathrm{dist}(x,\overline{\Att(M)})^2\geq 0$, one has
\[
V_x(M)\geq f(M), \qquad M\in\mathcal M.
\]
Set
\[
\alpha_x:=\inf_{M\in\mathcal M}V_x(M).
\]
By coercivity of $f$, there exists $R_\eta>0$ such that
\[
\|y\|\geq R_\eta \quad\Longrightarrow\quad f(y)\geq \alpha_x+\eta.
\]
Let
\[
\mathcal K_\eta:=\{M\in\mathcal M:\; M\cap \overline{\mathbb{B}}(0,R_\eta)\neq\varnothing\}.
\]
Since $\overline{\mathbb{B}}(0,R_\eta)$ is compact and only finitely many elements of $\mathcal M$
intersect a compact set, $\mathcal K_\eta$ is finite. Let $M\in\mathcal M\setminus \mathcal K_\eta$. Then $M\cap \overline{\mathbb{B}}(0,R_\eta)=\varnothing$,
so for every $m\in M$ one has $\|m\|>R_\eta$. Hence
$f(M)=f(m)\geq \alpha_x+\eta$. Therefore $V_x(M)\geq f(M)\geq \alpha_x+\eta$.
Taking the infimum over $M\in\mathcal M\setminus \mathcal K_\eta$ gives
\[
\inf_{M\in \mathcal M\setminus \mathcal K_\eta}V_x(M)\geq \alpha_x+\eta
=
\inf_{\widetilde M\in\mathcal M}V_x(\widetilde M)+\eta.
\]
\end{proof}
Any finite subset of $\mathcal{M}$ is automatically bounded by boundedness of every elements of $\mathcal{M}$ and they have a positive value gap with $\inf V_x$.
\begin{fact} \label{fact:1}
    Let $x\in \RR^d$, $\alpha_x := \inf_{M \in \mathcal M} V_x(M)$, $\varepsilon>0$, and \[
    \mathcal{M}_x = \bigcup_{M \in S_f^\gamma(x)}M,  \text{ and } (\mathcal{M}_x)_\varepsilon = \left\{ z \in \RR^d, \, \mathrm{dist}(z,\mathcal{M}_x)<\epsilon\right\}.
    \]
    Let $R>0$. On $\mathbb{B}(0,R)$ there exists a positive gap of $V_x$ between $M\in S_f^\gamma(x)$ and every other $M$ that intersects $\mathbb{B}(0,R)$, i.e., there exists $c_{R,\varepsilon}>0$ such that for every $M$ in
    \[
    M_{x,R,\varepsilon}:=\left\{ M \in \mathcal{M}, \, M \cap \mathbb{B}(0,R) \neq \emptyset, \, M\cap\left(\left(\mathcal{M}_x\right)_\varepsilon\right)^\complement\neq \emptyset \right\},
    \]
    \[
    V_x(M) \geq \alpha_x + c_{R,\epsilon}.
    \]
\end{fact}
\begin{lemma} \label{lm:Tf_bounded} Suppose that Assumptions \ref{ass:f_properties1}, \ref{ass:f_properties3}, \ref{ass:Tf}, and \ref{ass:Vx_compact_argmin} hold. Let $x\in \RR^d$ and $\delta,\gamma>0$.
    If one of the conditions below hold
    \begin{itemize}
        \item There exist $C_1,C_2,C_3>0$ and $0<\alpha<\beta$  such that 
    \begin{equation}
        (\forall z \in \RR^d), \qquad C_1\Vert z\Vert^\alpha -C_2\leq f(z) \leq C_3(1+\Vert z \Vert)^\beta. 
    \end{equation}
    \item $\mathcal{M}$ has finitely many components.
    \end{itemize}
    Then, $T_f:\RR^d \to \mathcal{M}_\ast$ is integrable and square integrable against the Gaussian distribution, i.e.,
    \begin{equation}
        \int_{\RR^d} \Vert T_f(y) \Vert e^{\frac{-1}{2\gamma\delta} \Vert y - x \Vert^2} dy <+\infty, \text{ and } \int_{\RR^d} \Vert T_f(y) \Vert^2 e^{\frac{-1}{2\gamma\delta} \Vert y - x \Vert^2} dy <+\infty
    \end{equation}
\end{lemma}
\begin{proof}
    Let us consider the second case first. Denote by $J$ the number of components of $\mathcal{M}$. Then $\mathcal{M}_\ast = \bigcup_{j=1}^J \mathcal{M}_j$. By assumption any $M \in \mathcal{M}$ is bounded. Thus there exists some $R>0$ such that $\mathcal{M}_\ast \subset \mathbb{B}(0,R)$. Hence,
    \begin{align*}
        \int_{\RR^d} \Vert T_f(y) \Vert e^{-\frac{\Vert y-x \Vert^2}{2\delta \gamma}}dy \leq R \int_{\RR^d} e^{-\frac{\Vert y-x \Vert^2}{2\delta \gamma}}dy <+\infty.
    \end{align*}
    Now to the first case. We have for all $z \in \RR^d$, $f(z) \geq f(T_f(z))$, therefore
    \[
     C_1 \Vert T_f(z) \Vert^\alpha -C_2 \leq f(T_f(z)) \leq f(z) \leq C_3(1+\Vert z \Vert)^\beta
    \]
    and,
    \[
     \Vert T_f(z) \Vert \leq C_1^{-1/\alpha}\left(C_2 + C_3(1 + \Vert z \Vert^\beta)\right)^{1/\alpha}.
    \]
    Hence,
    \[
    \int_{\RR^d} \Vert T_f(y) \Vert e^{-\frac{\Vert y-x \Vert^2}{2\delta \gamma}}dy \leq \int_{\RR^d} C_1^{-1/\alpha}\left(C_2 + C_3(1 + \Vert y \Vert^\beta)\right)^{1/\alpha} e^{-\frac{\Vert y-x \Vert^2}{2\delta \gamma}}dy < +\infty.
    \]
    The square case follows immediately from the same arguments.
\end{proof}
Assumption \ref{ass:growth} and the first condition in Lemma \ref{lm:Tf_bounded} imply that $f$ has polynomially bounded growth. This is slightly stronger than coercivity and is akin to the notion of radially unbounded function in control theory\footnote{In this setting, $\psi$ is also required to be positive \citep[Section 2]{kellett2014compendium}.} \citep{kellett2014compendium}.
\begin{lemma} \label{lm:far_tail1}Suppose that Assumptions \ref{ass:f_properties1}, \ref{ass:Tf}, \ref{ass:Vx_compact_argmin}, and \ref{ass:growth} hold. Let $x\in \RR^d$, $Y_\delta \sim \mathcal{N}(x,\gamma \delta \Id)$, $q>0$ and $r>0$. We have:
    \begin{equation}
        \limsup_{\delta \downarrow 0} \delta \log \mathbb{E}\left[(1 + \Vert Y_\delta \Vert^q)\mathbf{1}_{\{ \Vert Y_\delta \Vert > r \}} \right] \leq -\frac{(r- \Vert x \Vert^2)_+}{2\gamma}.
    \end{equation}
\end{lemma}
\begin{proof}
We have
\[
\mathbb{E}\!\left[(1+\|Y_\delta\|^q)\mathbf 1_{\{\|Y_\delta\|>r\}}\right] = \frac{1}{(2\pi\delta\gamma)^{d/2}} \int_{\{\|y\|>r\}} (1+\|y\|^q)\exp\!\left(-\frac{\|y-x\|^2}{2\gamma\delta}\right)\,dy.
\]
Fix $R>r$. Splitting the domain gives $\{\|y\|>r\}=\{r<\|y\|\leq R\}\cup \{\|y\|>R\}$, hence we can split the expectation in
\[
I_\delta(r):=
\mathbb{E}\!\left[(1+\|Y_\delta\|^q)\mathbf 1_{\{\|Y_\delta\|>r\}}\right]
=I_{\delta,1}(r,R)+I_{\delta,2}(R),
\]
where
\[
I_{\delta,1}(r,R):=
\frac{1}{(2\pi\delta\gamma)^{d/2}}
\int_{\{r<\|y\|\leq R\}}
(1+\|y\|^q)e^{-\frac{\|y-x\|^2}{2\gamma\delta}}\,dy
\]
and
\[
I_{\delta,2}(R):=
\frac{1}{(2\pi\delta\gamma)^{d/2}}
\int_{\{\|y\|>R\}}
(1+\|y\|^q)e^{-\frac{\|y-x\|^2}{2\gamma\delta}}\,dy.
\]

On the bounded region $\{r<\|y\|\leq R\}$, we have $1+\|y\|^q\leq 1+R^q$. Moreover, if $\|y\|>r$, then $
\|y-x\|\geq \|y\|-\|x\|\geq (r-\|x\|)_+.
$
Therefore
\[
I_{\delta,1}(r,R)
\leq \frac{(1+R^q)\,\lambda_d(\mathbb \mathbb{B}(0,R))}{(2\pi\delta\gamma)^{d/2}} \exp\!\left(-\frac{(r-\|x\|)_+^2}{2\gamma\delta}\right).
\]
Since $\delta\log(\delta^{-d/2})\to 0$, it follows that
\[
\limsup_{\delta\downarrow 0}\delta\log I_{\delta,1}(r,R)
\leq -\frac{(r-\|x\|)_+^2}{2\gamma}.
\]

Now onto $I_{\delta,2}(R)$. Assume $R>\|x\|$. For $\|y\|>R$, $\|y-x\|\geq R-\|x\|,$
hence $
e^{-\frac{\|y-x\|^2}{2\gamma\delta}}
\leq e^{-\frac{(R-\|x\|)^2}{4\gamma\delta}} e^{-\frac{\|y-x\|^2}{4\gamma\delta}}.
$
Thus
\[
I_{\delta,2}(R) \leq e^{-\frac{(R-\|x\|)^2}{4\gamma\delta}} \frac{1}{(2\pi\delta\gamma)^{d/2}} \int_{\mathbb R^d}(1+\|y\|^q)e^{-\frac{\|y-x\|^2}{4\gamma\delta}}\,dy.
\]
Using $\|y\|^q\leq \max\{1,2^{q-1}\}(1+\|y-x\|^q+\|x\|^q)$, we obtain $1+\|y\|^q\leq C(1+\|y-x\|^q)$
for some constant $C>0$. Hence
\[
I_{\delta,2}(R)
\leq C e^{-\frac{(R-\|x\|)^2}{4\gamma\delta}}
\frac{1}{(2\pi\delta\gamma)^{d/2}}
\int_{\mathbb R^d} (1+\|y-x\|^q)e^{-\frac{\|y-x\|^2}{4\gamma\delta}}\,dy.
\]
Making the change of variables $z=(y-x)/\sqrt{\delta}$, we obtain that the righthandside integral is in $O(\delta^{d/2}) + O(\delta^{(d+q)/2})$ hence bounded uniformly for $\delta\in(0,1]$. Therefore there exists $C_R>0$ such that
\[
I_{\delta,2}(R)\leq C_R e^{-\frac{(R-\|x\|)^2}{4\gamma\delta}},
\]
and thus
\[
\limsup_{\delta\downarrow 0}\delta\log I_{\delta,2}(R)
\le
-\frac{(R-\|x\|)^2}{4\gamma}.
\]
Finally,
\[
\limsup_{\delta\downarrow 0}\delta\log I_\delta(r)
\le
\max\!\left\{
\limsup_{\delta\downarrow 0}\delta\log I_{\delta,1}(r,R),\,
\limsup_{\delta\downarrow 0}\delta\log I_{\delta,2}(R)
\right\}.
\]
Hence
\[
\limsup_{\delta\downarrow 0}\delta\log I_\delta(r)
\le
\max\!\left\{
-\frac{(r-\|x\|)_+^2}{2\gamma},\,
-\frac{(R-\|x\|)^2}{4\gamma}
\right\}.
\]
Choosing $R>r$ large enough so that
\[
\frac{(R-\|x\|)^2}{4\gamma}>
\frac{(r-\|x\|)_+^2}{2\gamma},
\]
we conclude that
\[
\limsup_{\delta\downarrow 0}\delta\log \mathbb{E}\!\left[(1+\|Y_\delta\|^q)\mathbf 1_{\{\|Y_\delta\|>r\}}\right]
\le
-\frac{(r-\|x\|)_+^2}{2\gamma}.
\qedhere
\]
\end{proof}
\begin{remark} \label{remark:2}
    Evolution of the set $\mathcal{M}_x$ with respect to $\gamma$. There exists approximately three regimes. When $\gamma$ is small, $V_x$ heavily penalizes distant local minimizers hence its $\argmin$ is more likely to contain one unique solution. Slowly increasing $\gamma$ makes it more likely that several competitors $M$ can arise. Then as $\gamma$ increases more, $V_x$ is more likely to select a unique better local minimizer.
\end{remark}
The first regime is characterized in the next Lemma.
\begin{lemma}
    Suppose that Assumptions \ref{ass:f_properties1}, \ref{ass:f_properties3}, \ref{ass:Tf}, \ref{ass:graph}, and \ref{ass:Vx_compact_argmin} hold. Let $x\in \RR^d$, $\gamma>0$, and $P\in\mathcal{M}$ the closest (w.r.t. to $\Att(P)$) minimizing component. Denote by $Q$ the closest (w.r.t. $\Att(Q)$) better local minimizing component of $P$ in the sense of Assumption \ref{ass:graph}. If:
    \begin{equation}
        \gamma < \frac{\mathrm{dist}(x,\overline{\Att(Q)})^2-\mathrm{dist}(x,\overline{\Att(P)})^2}{2(f(P)-f(Q))}
    \end{equation}
    then,
    \begin{equation}
        P = \argmin_{M \in \mathcal{M}} V_x(M).
    \end{equation} 
\end{lemma}
\begin{proof}
    We have that for all $M\in \mathcal{M}\setminus \{Q\}$ such that $\mathrm{dist}(x,\overline{\Att(M)})\leq d_f$ then $f(M)>f(Q)$. Furthermore, by Assumption \ref{ass:graph}, any better local minimizer in the vicinity of $P$ has function value equal to that of $Q$ in a radius $d_f$, hence we just need to make sure that $Q$ does not belong in the $\argmin$, giving immediately the positive upper bound on $\gamma$.
\end{proof}
These scenarios where we have a series of equally distant and of equal function values are not possible indefinitely, first because there is a unique global minimizer, and second because the function $f$ is coercive so it needs to grow to infinity, thus contradicting the existence of these increasingly distant local minimizers of decreasing function values.

\section{Proofs of Section 3: Proposed method} \label{app:proofs_sec3}

\subsection{Proofs for ideal basin hopping} \label{app:proofs_ideal_pbh}
\paragraph{Proof of Lemma \ref{lm:Tf_prox}}
\begin{proof}
    We have for all $x\in \RR^d$ 
\begin{align*}
    \prox_{\gamma f \circ T_f}(x) & = \argmin_{y \in \RR^d} \frac{1}{2\gamma} \Vert y-x \Vert^2 + f (T(y)),
\end{align*}
and
\[
S_f^\gamma(x) := \argmin_{M \in \mathcal{M}} \frac{1}{2\gamma} \mathrm{dist}(x,\overline{\Att(M)})^2 + f(M).
\]
Looking more closely at the first optimization problem, we have
\begin{align*}
    \min_{y \in \RR^d} \frac{1}{2\gamma} \Vert y-x \Vert^2 + f (T(y)) = \min_{M \in \mathcal{M}} \frac{1}{2\gamma} \mathrm{dist}(x,\overline{\Att(M)})^2 + f(M).
\end{align*}
Hence, if for every
$
M\in S_f^\gamma(x), \, \mathrm{proj}_{\overline{\Att(M)}}(x)\cap \Att(M)\neq\varnothing,
$
\begin{equation*}
    \prox_{\gamma f \circ T_f}(x) = \bigcup_{ M \in \argmin_{M \in \mathcal{M}} \frac{1}{2\gamma} \mathrm{dist}(x,\overline{\Att(M)})^2 + f(M) } \mathrm{proj}_{\overline{\Att(M)}}(x),
\end{equation*}
where $\mathrm{proj}$ denotes the orthogonal projection. Thus, as $T_f$ maps to $\mathcal{M}$, we recover
\begin{equation*}
     T_f(\prox_{\gamma f \circ T_f}(x)) = \argmin_{M \in \mathcal{M}} \frac{1}{2\gamma} \mathrm{dist}(x,\overline{\Att(M)})^2 + f(M). \qedhere
\end{equation*}
\end{proof}
\paragraph{Proof of Lemma \ref{lm:gamma_ast}}
\begin{proof}
    Let $x\in \RR^d$. We have that for all $\gamma>0$.
    \begin{equation*}
        S_f^\gamma(x) = \argmin_{M \in \mathcal{M}} \frac{1}{2\gamma} \mathrm{dist}(x,\overline{\Att(M)})^2 + f(M).
    \end{equation*}
    Hence, 
    \begin{align*}
        x_\ast \in S_f^\gamma(x)& \Leftrightarrow  \forall M \in \mathcal{M}, \,  \frac{1}{2\gamma} \mathrm{dist}(x,\overline{\Att(x_\ast)})^2 + f(x_\ast) \leq \frac{1}{2\gamma} \mathrm{dist}(x,\overline{\Att(M)})^2 + f(M)\\
        & \Leftrightarrow  \forall M \in \mathcal{M}, \,  \frac{1}{2\mu}\left[ \mathrm{dist}(x,\overline{\Att(x_\ast)})^2 - \mathrm{dist}(x,\overline{\Att(M)})^2 \right] \leq \gamma,
    \end{align*}
    where the last inequality was obtained using Assumption \ref{ass:f_properties3}.
    The minimal required value of $\gamma$ depends on the close and achievable minimizers. Indeed, all minimizing components such that $\mathrm{dist}(x,\overline{\Att(x_\ast)})^2 - \mathrm{dist}(x,\overline{\Att(M)})^2\leq 0$ yields no constraint on $\gamma$.
\end{proof}
\paragraph{Proof of Theorem \ref{th:ideal_pbh}.}
\begin{proof}
    First iterations are well defined as $S_f^\gamma$ is non-empty and bounded (Assumption \ref{ass:Vx_compact_argmin}, Lemma \ref{lm:Tf_prox}). Second, if for any $K\in \mathbb{N}$, $x_K = x_\ast$ then $x_{k} = x_\ast$ for all $k>K$. Finally, by construction $f(x_{k+1})\leq f(x_k)$, hence by coercivity of $f$ (Assumption \ref{ass:f_properties1}), the iterates are bounded. Denote by $B_0$ this bounded set. %
    We can guarantee that there exists only finitely many failed iterations (where $f(x_+)=f(x_{k})$).
    
    Indeed, by boundedness of the iterates, there exists $d_f\geq 0$ such that for all $k \in \mathbb{N}$,
    \[
    \mathrm{dist}(x_k,\overline{\Att(x_\ast)}) \leq d_f. 
    \]
    Let $P \in \mathcal{M}$, such that $P \cap B_0 \neq \emptyset$ and $P \neq \{x_\ast\}$.
    We have for any $x \in P$, $V_x(x_\ast) < V_x(P)$ if
    \[
    \frac{\mathrm{dist}(x,\overline{\Att(x_\ast)})^2}{2\gamma} + f_{\min} < f(P)
    \]
    which happens if $
    \gamma > \frac{d_f^2}{2(f(P)-f_{\min})}.
    $
    Thus, taking the infimum over all $P$ not equal to $\{x_\ast\}$, we have $\gamma_s \geq \frac{d_f^2}{2\mu}$ implies that for all $x \in B_0$, $x_\ast \in S_f^{\gamma_s}(x)$.
    Let $K\in \mathbb{N}$. For any $\gamma>0$,
    \begin{align*}
        & \gamma \eta^K \leq \gamma_s \\
        \Leftrightarrow & K \log(\eta) \leq \log\left(\frac{\gamma_s}{\gamma}\right) \\
        \Leftrightarrow & K \leq \frac{\log\left(\gamma_s/\gamma\right)}{\log(\eta)}
    \end{align*} 
    Therefore, after at most
    $\left\lceil \frac{\log\left(\gamma_s/\gamma\right)}{\log(\eta)}\right\rceil_+
    $ 
    failed iterations, $x_\ast$ belongs to $S_f^\gamma$. 

    Now, by boundedness of the iterates, only finitely many $M \in \mathcal{M}$ can be visited. Hence, we conclude that convergence occurs in a finite number of iterations.
\end{proof}
\subsection{Proofs for exact expectation basin hopping}
\label{app:proofs_exact_expect_pbh}

\paragraph{Proof of Lemma \ref{lm:delta_Fm}.}
\begin{proof}
By the small-noise Gaussian large deviation principle with speed $\delta^{-1}$, and rate function $I_x(y)=\frac{1}{2\gamma}\|y-x\|^2$, for every Borel set $A\subset \mathbb{R}^d$ \citep{denhollander2000large}, \footnote{The small-noise Gaussian rate function can be identified directly from the explicit normal example in \citep[pp.4--5]{kosygina_mountford_ldp_notes} with  $X_n\sim\mathcal N(0,1/n)$ , which yields the rate function $I(x)=x^2/2$, and cf. \citep[Exercise III.9, p.30]{denhollander2000large}.}
\begin{equation*}
-\inf_{y\in A^\circ} I_x(y)
\leq
\liminf_{\delta\downarrow 0}\delta \log \mathbb{P}(Y_\delta\in A)
\leq
\limsup_{\delta\downarrow 0}\delta \log \mathbb{P}(Y_\delta\in A)
\leq
-\inf_{y\in \overline{A}} I_x(y).
\end{equation*}
Applying this with $A=\Att(M)$ yields
\begin{align*}
-\inf_{y\in \Att(M)^\circ}\frac{1}{2\gamma}\|y-x\|^2
& \leq
\liminf_{\delta\downarrow 0}\delta \log \mathbb{P}(Y_\delta\in \Att(M)) \\
& \leq
\limsup_{\delta\downarrow 0}\delta \log \mathbb{P}(Y_\delta\in \Att(M))
\leq
-\inf_{y\in \overline{\Att(M)}}\frac{1}{2\gamma}\|y-x\|^2.
\end{align*}
Which is equivalent to
\begin{align*}
-\frac{1}{2\gamma}\mathrm{dist}(x,\Att(M)^\circ)^2
& \leq
\liminf_{\delta\downarrow 0}\delta \log \mathbb{P}(Y_\delta\in \Att(M))
\\ & \leq
\limsup_{\delta\downarrow 0}\delta \log \mathbb{P}(Y_\delta\in \Att(M))
\leq
-\frac{1}{2\gamma}\mathrm{dist}(x,\overline{\Att(M)})^2.
\end{align*}
By the regularity assumption the lower and upper bounds coincide, and hence
\begin{equation*}
\lim_{\delta\downarrow 0}\delta \log \mathbb{P}(Y_\delta\in \Att(M))
=
-\frac{1}{2\gamma}\mathrm{dist}(x,\overline{\Att(M)})^2.
\end{equation*}
This proves the first claim. The equivalent exponential form follows immediately: if a scalar quantity $a_\delta$ satisfies $\delta \log a_\delta \to -c$,
then $ a_\delta = \exp \left(-\frac{1}{\delta}[c+o(1)]\right)$ (see \citep[Definition 401]{shalizi2006largedeviations}).
\end{proof} 

\paragraph{First case: finite number of minimizers.} If there are finitely many sets of local minimizers, then we can obtain an explicit concentration result inside the convex hull defined by the solutions in $\mathcal{M}_x$ (Fact \ref{fact:1}), when $\delta$ goes to $0$\footnote{This result is similar in spirit to \citep[Corollary 1]{zhang2024inexact} even though proofs are completely different}.
\begin{proposition} \label{prop:concentration_finite} Let $x\in\mathbb{R}^d$, $\gamma>0$, $\delta>0$, and let $
Y_\delta \sim \mathcal{N}(x,\delta\gamma I_d)$.
    Suppose that $\mathcal{M}$ has finitely many components, i.e.,
    \begin{equation}
        \mathcal{M}_\ast = \bigcup_{1\leq j \leq J} M_j
    \end{equation}
    for some $J \in \mathbb N$. 
    Denote by $I_x$ the set of indices of the minimizers of $\left\{f_j + \frac{1}{2\gamma}\mathrm{dist}(x,\Att(M_j))^2\right\}_{1\leq j \leq J}$, i.e., $I_x \subseteq \left\{1,\ldots,J\right\}$. Then, 
    \begin{equation*}
        \frac{\mathbb{E}[T_f(Y_\delta)e^{-f(T_f(Y_\delta))/\delta}]}{\mathbb{E}[e^{-f(T_f(Y_\delta))/\delta}]} \xrightarrow[\delta\downarrow 0]{} \frac{\sum_{i \in I_x} e^{-f_i/\delta} \mathbb{P}(Y_\delta \in \Att(M_i)) \mathbb{E}[T_f(Y_\delta) | Y_\delta\in \Att(M_i) ]}{\sum_{i \in I_x}e^{-f_i/\delta} \mathbb{P}(Y_\delta \in \Att(M_i))}.
    \end{equation*}
    If $I_x$ is a singleton $\{i_x\}$, and $M_{i_x}=\{m_{i_x}\}$ contains an isolated local minimizer then, 
    \begin{equation*}
        \frac{\mathbb{E}[T_f(Y_\delta)e^{-f(T_f(Y_\delta))/\delta}]}{\mathbb{E}[e^{-f(T_f(Y_\delta))/\delta}]} \xrightarrow[\delta\downarrow 0]{} m_{i_x}.
    \end{equation*}
\end{proposition}

\begin{proof}
Since the local minimizers reached by $T_f$ decompose into the finite disjoint union $\mathcal{M}$, the sets $(\Att_j)_{1\leq j \leq J}$ form a measurable partition of the subset of $\RR^d$ on which $T_f$ will land in $\mathcal{M}$. Denoting $f_j$ the value of $f$ attained on $M_j$, we have
\begin{align*}
    \mathbb{E}[T_f(Y_\delta)e^{-f(T_f(Y_\delta))/\delta}]& =
\sum_{j=1}^J
\mathbb{E}[T_f(Y_\delta)e^{-f(T_f(Y_\delta))/\delta}\mathbf 1_{\{Y_\delta\in \Att(M_j)\}}] \\
& = \sum_{j=1}^J e^{-f_j/\delta}
\mathbb{E}[T_f(Y_\delta)\mathbf 1_{\{Y_\delta\in \Att(M_j)\}}] \\
& = \sum_{j=1}^J e^{-f_j/\delta} \mathbb{P}(Y_\delta \in \Att(M_j))
\mathbb{E}[T_f(Y_\delta) | Y_\delta\in \Att(M_j) ].
\end{align*}
Moreover, using Lemma \ref{lm:delta_Fm}, we obtain that
\begin{align*}
   \delta \log \left( e^{-f_j/\delta} \mathbb{P}(Y_\delta \in \Att(M_j)) \right) & =_{\delta \to 0} - f_j - \frac{1}{2\gamma}\mathrm{dist}(x,\Att(M_j))^2,
\end{align*}
implying 
\begin{equation*}
    e^{-f_j/\delta} \mathbb{P}(Y_\delta \in \Att(M_j)) =_{\delta \to 0} \exp\left(-\frac{1}{\delta}\left[f_j + \frac{1}{2\gamma}\mathrm{dist}(x,\Att(M_j))^2 + o(1)\right]\right).
\end{equation*} 
Now, take the set $I_x$. For all $i \in I_x$, and $j \in J\setminus I_x$, we have
\begin{equation*}
    \delta \log \left( \frac{e^{-f_i/\delta} \mathbb{P}(Y_\delta \in \Att(M_i))}{e^{-f_j/\delta} \mathbb{P}(Y_\delta \in \Att(M_j))} \right) \xrightarrow[\delta\downarrow 0]{} f_j + \frac{1}{2\gamma} \mathrm{dist}(x,\Att(M_j))^2 - f_i - \frac{1}{2\gamma}\mathrm{dist}(x,\Att(M_i))^2 > 0
\end{equation*}
implying that
\begin{equation*}
    \lim_{\delta \to 0}\frac{e^{-f_i/\delta} \mathbb{P}(Y_\delta \in \Att(M_i))}{e^{-f_j/\delta} \mathbb{P}(Y_\delta \in \Att(M_j))} =+\infty.
\end{equation*}
Hence, we obtain that
\begin{equation*}
    \lim_{\delta \to 0}\frac{\sum_{j \in J\setminus I_x} e^{-f_j/\delta} \mathbb{P}(Y_\delta \in \Att(M_j))}{\sum_{j \in J} e^{-f_j/\delta} \mathbb{P}(Y_\delta \in \Att(M_j))} = 0.
\end{equation*}
The rest follows from the fact that no connected components of $\mathcal{M}$ is unbounded, guaranteeing that $\mathbb{E}[T_f(Y_\delta) | Y_\delta\in \Att(M_j) ]$ is bounded for all $j$.
\end{proof}

This result highlights the need for the uniqueness of the global minimizer: the barycenter of several global minimizers does not make any sense unless these global minimizers belong to the same connected and convex set.
\paragraph{Proof of Proposition \ref{prop:measure_concentration}.}
\begin{proof}
Fix $\eta>0$, and let $K_\eta\subset \mathcal{M}_\ast$ be a bounded set such that (Lemma \ref{lm:Vx})
\[
\inf_{m\notin K_\eta}V_x(m)\geq \alpha_x+\eta.
\]
Set
\[
A_\eta:=\{y\in\mathbb{R}^d:\ T_f(y)\in K_\eta^\complement\}.
\]
We look at our measure outside this bounded set
\begin{align*}
\nu_{x,\delta}(K_\eta^\complement)
& =
\mathbb{E} \left[
e^{-f(T_f(Y_\delta))/\delta}\mathbf 1_{\{T_f(Y_\delta)\in K_\eta^\complement\}}
\right] \\
& = \frac{1}{(2\pi\delta\gamma)^{d/2}}
\int_{A_\eta}
\exp \left(
-\frac{1}{\delta}
\left[
f(T_f(y))+\frac{\|y-x\|^2}{2\gamma}
\right]
\right)\,dy.
\end{align*}
We first study the term in the exponential. Set $g(y) = f(T_f(y))+\frac{\Vert y-x \Vert^2}{2\gamma}$. 
If $y\in A_\eta$, then $T_f(y)\in M\in K_\eta^\complement$, and since $y\in \Att(M)$,
\[
\frac{\|y-x\|^2}{2\gamma}\geq \frac{\mathrm{dist}(x,\overline{\Att(M)})^2}{2\gamma}.
\]
Hence,
\[
g(y) = f(M)+\frac{\|y-x\|^2}{2\gamma}
\geq
f(M)+\frac{\mathrm{dist}(x,\overline{\Att(M)})^2}{2\gamma} = V_x(M) \geq \alpha_x+\eta.
\]
Therefore, $\inf_{y\in A_\eta}g(y)\geq \alpha_x+\eta.$
Now, to estimate the measure, we split our set with a ball of radius $R>0$ in the following way:
\[
A_\eta=\left(A_\eta\cap \mathbb{B}(0,R) \right)\cup \left(A_\eta\cap \mathbb{B}(0,R)^\complement\right).
\]
The measure on the lefthandside of the union is easily bounded by the previous inequality as we have,
\[
\frac{1}{(2\pi\delta\gamma)^{d/2}}
\int_{A_\eta\cap \mathbb{B}(0,R)} e^{-g(y)/\delta}\,dy
\leq
\frac{\lambda_d(\mathbb{B}(0,R))}{(2\pi\delta\gamma)^{d/2}}e^{-(\alpha_x+\eta)/\delta}.
\]
Since
\[
\delta\log \left(\delta^{-d/2}\right)\xrightarrow[\delta\downarrow 0]{} 0,
\]
it follows that
\[
\limsup_{\delta\downarrow 0}
\delta\log\left[\frac{1}{(2\pi\delta\gamma)^{d/2}}
\int_{A_\eta\cap \mathbb{B}(0,R)} e^{-g(y)/\delta}\,dy\right] \leq -(\alpha_x+\eta).
\]
$f$ is bounded from below on $\mathbb R^d$, hence all $y$,
\[
g(y)\geq f_{\min}+\frac{\|y-x\|^2}{2\gamma}.
\]
Thus
\[
\frac{1}{(2\pi\delta\gamma)^{d/2}}
\int_{A_\eta\cap \mathbb{B}(0,R)^\complement} e^{-g(y)/\delta}\,dy
\leq
e^{-f_{\min}/\delta} \frac{1}{(2\pi\delta\gamma)^{d/2}}\int_{\mathbb{B}(0,R)^\complement} \exp \left(-\frac{\|y-x\|^2}{2\gamma\delta}\right)\,dy.
\]
The last term is exactly
\[
e^{-f_{\min}/\delta}\,\mathbb P(\|Y_\delta\|>R).
\]
We can take $R>\|x\|$, large enough so that
\[
f_{\min}+\frac{(R-\|x\|)^2}{2\gamma}>\alpha_x+\eta.
\]
Since $\|y-x\|\geq R-\|x\|$ whenever $\|y\|>R$, the Gaussian tail satisfies \citep{denhollander2000large}
\[
\limsup_{\delta\downarrow 0}\delta\log \mathbb P(\|Y_\delta\|>R)\leq-\frac{(R-\|x\|)^2}{2\gamma},
\]
and therefore
\[\limsup_{\delta\downarrow 0}\delta\log\left[e^{-f_{\min}/\delta}\mathbb P(\|Y_\delta\|>R)\right]\leq
-f_{\min}-\frac{(R-\|x\|)^2}{2\gamma}<-(\alpha_x+\eta).
\]

Combining the estimates on $A_\eta\cap \mathbb{B}(0,R)$ and $A_\eta\cap \mathbb{B}(0,R)^\complement$, we obtain
$
\limsup_{\delta\downarrow 0}\delta\log \nu_{x,\delta}(K_\eta^\complement) \leq -(\alpha_x+\eta).
$
We now lower bound the denominator. Let $\varepsilon>0$. By definition of $\alpha_x$, there exists $M_\varepsilon\in\mathcal{M}$ such that
\[
V_x(M_\varepsilon)\leq \alpha_x+\varepsilon.
\]
Then
\[
\nu_{x,\delta}(\mathcal{M})
=
\mathbb E \left[e^{-f(T_f(Y_\delta))/\delta}\right]
\geq
e^{-f(m_\varepsilon)/\delta}\mathbb P(Y_\delta\in \Att(M_\varepsilon)).
\]
By Lemma \ref{lm:delta_Fm}, 
\[
\lim_{\delta\downarrow 0}
\delta\log \mathbb P(Y_\delta\in \Att(M_\varepsilon))
=
-\frac{1}{2\gamma}\mathrm{dist}(x,\Att(M_\varepsilon))^2,
\]
hence
\[
\liminf_{\delta\downarrow 0}
\delta\log \nu_{x,\delta}(\mathcal{M})
\geq
-f(M_\varepsilon)-\frac{1}{2\gamma}\mathrm{dist}(x,\Att (M_\varepsilon))^2
=
-V_x(M_\varepsilon)
\geq
-(\alpha_x+\varepsilon).
\]
Letting $\varepsilon\downarrow 0$ gives
\[
\liminf_{\delta\downarrow 0}
\delta\log \nu_{x,\delta}(\mathcal{M})
\geq
-\alpha_x.
\]

Finally,
\[
\mu_{x,\delta}(K_\eta^\complement)
=
\frac{\nu_{x,\delta}(K_\eta^\complement)}{\nu_{x,\delta}(\mathcal{M})},
\]
so
\[
\limsup_{\delta\downarrow 0}
\delta\log \mu_{x,\delta}(K_\eta^\complement)
\le
-(\alpha_x+\eta)-(-\alpha_x)
=
-\eta.
\]
In particular, $\mu_{x,\delta}(K_\eta^\complement)\to 0$,
that is, $\mu_{x,\delta}(K_\eta)\to 1$.
\end{proof}

\paragraph{Proof of Proposition \ref{prop:gaussian_tail_poly}.}
\begin{proposition}\label{prop:gaussian_tail_poly}
Let $x\in \mathbb R^d$, $\gamma>0$, $q\geq 0$, and let
$Y_\delta \sim \mathcal N(x,\delta\gamma I_d)$. 
Let $R>0$. We have
\begin{equation*}
    \limsup_{\delta \downarrow 0} \delta \log \left(\mathbb{E}[\Vert T_f(Y_\delta)\Vert e^{-f(T_f(Y_\delta))/\delta}\mathbf{1}_{\{\Vert T_f(Y_\delta)\Vert>R\}}]\right) \leq -f_{\min} -\frac{\left(\left( \frac{C_1 R^\alpha-C_2-C_3}{C_3}\right)_+^{1/\beta}-\Vert x \Vert^2\right)_+}{2\gamma},
\end{equation*}
and that 
\begin{equation*}
    \liminf_{\delta\downarrow 0} \delta \log \mathbb{E}[e^{-f(T_f(Y_\delta))/\delta}] \geq -\alpha_x.
\end{equation*}
Then,
\begin{equation*}
     \limsup_{\delta \downarrow 0} \delta \log \left(\frac{\mathbb{E}[\Vert T_f(Y_\delta)\Vert e^{-f(T_f(Y_\delta))/\delta}\mathbf{1}_{\{\Vert T_f(Y_\delta)\Vert>R\}}]}{\mathbb{E}[e^{-f(T_f(Y_\delta))/\delta}]} \right) \leq \alpha_x-f_{\min} -\frac{\left(\left( \frac{C_1 R^\alpha-C_2-C_3}{C_3}\right)_+^{1/\beta}-\Vert x \Vert^2\right)_+}{2\gamma}.
\end{equation*}
\end{proposition}

\begin{proof}
    We have that 
    \begin{equation*}
         \frac{\mathbb{E}[T_f(Y_\delta)e^{-f(T_f(Y_\delta))/\delta}]}{\mathbb{E}[e^{-f(T_f(Y_\delta))/\delta}]}
    \end{equation*}
    is well defined. Indeed, $0\leq\mathbb{E}[e^{-f(T_f(Y_\delta))/\delta}] \leq e^{-f_{\min}/\delta} < +\infty$ and 
    $
    \mathbb{E}[\Vert T_f(Y_\delta)\Vert e^{-f(T_f(Y_\delta))/\delta}] \leq e^{-f_{\min}/\delta}\mathbb{E}[\Vert T_f(Y_\delta)\Vert] < +\infty
    $ by Lemma \ref{lm:Tf_bounded}. Using our measure introduced in Proposition \ref{prop:measure_concentration} we have
    \begin{equation*}
        \frac{\mathbb{E}[T_f(Y_\delta)e^{-f(T_f(Y_\delta))/\delta}]}{\mathbb{E}[e^{-f(T_f(Y_\delta))/\delta}]} =\int_{m\in \mathcal{M}_\ast} m d\mu_{x,\delta}(m) 
    \end{equation*}
    Using similar arguments than in the proof of Lemma \ref{lm:Tf_bounded} we can have that $m\mapsto \Vert m \Vert$ is integrable against $\mu_{x,\delta}$. We have
    \[
        \int_{m\in \mathcal{M}_\ast, \Vert m \Vert > R} \Vert m\Vert  d\mu_{x,\delta}(m) = \frac{\mathbb{E}[\Vert T_f(Y_\delta)\Vert e^{-f(T_f(Y_\delta))/\delta}\mathbf{1}_{\{\Vert T_f(Y_\delta)\Vert>R\}}]}{\mathbb{E}[e^{-f(T_f(Y_\delta))/\delta}]}
    \]
    As
    \[ 
    \mathbb{E}[\Vert T_f(Y_\delta)\Vert e^{-f(T_f(Y_\delta))/\delta}\mathbf{1}_{\Vert T_f(Y_\delta)\Vert>R}] \leq e^{-f_{\min}/\delta}\mathbb{E}[C_1^{-1/\alpha}\left(C_2 +C_3(1+\Vert Y_\delta\Vert^\beta)\right)^{1/\alpha}\mathbf{1}_{\Vert T_f(Y_\delta)\Vert>R}],
    \]
    and, 
    \begin{align*}
        & R < C_1^{-1/\alpha}\left(C_2 +C_3(1+\Vert Y_\delta\Vert^\beta)\right)^{1/\alpha} \\
        \Leftrightarrow & \left(\frac{C_1 R^\alpha-C_2-C_3}{C_3}\right)^{1/\beta}< \Vert Y_\delta\Vert,
    \end{align*}
    implying that 
    \[
    \left\{ \Vert T_f(Y_\delta) \Vert >R \right\} \subset \left\{ \Vert Y_\delta \Vert >  \left( \frac{(C_1 R^\alpha-C_2-C_3)_+}{C_3}\right)^{1/\beta} \right\}.
    \]
    Therefore,
    \begin{align*}
        &\mathbb{E}[\Vert T_f(Y_\delta)\Vert e^{-f(T_f(Y_\delta))/\delta}\mathbf{1}_{\Vert T_f(Y_\delta)\Vert>R}]
        \\ & \leq e^{-f_{\min}/\delta}\mathbb{E}\left[C_1^{-1/\alpha}\left(C_2 +C_3(1+\Vert Y_\delta\Vert^\beta)\right)^{1/\alpha}\mathbf{1}_{\left\{ \Vert Y_\delta \Vert >  \left( \frac{R^\alpha-C_2-C_3}{C_3}\right)^{1/\beta} \right\}}\right]
    \end{align*}
    We can show that there exists $C_{2,\alpha},C_{3,\alpha}>0$ such that
    \[
    C_1^{-1/\alpha}\left(C_2 +C_3(1+\Vert Y_\delta\Vert^\beta)\right)^{1/\alpha} \leq C_{2,\alpha} + C_{3,\alpha} (1+\Vert Y_\delta \Vert^{\beta/\alpha}).
    \]
    Now, using Lemma \ref{lm:far_tail1} with $q=\beta/\alpha$ and $r= \left( \frac{C_1 R^\alpha-C_2-C_3}{C_3}\right)^{1/\beta}$ we have
    \begin{align*}
         \limsup_{\delta \downarrow 0} \delta \log \mathbb{E}\left[C_1^{-1/\alpha}\left(C_{2} +C_{3}(1+\Vert Y_\delta\Vert^\beta)\right)^{1/\alpha}\mathbf{1}_{\left\{ \Vert Y_\delta \Vert >  r \right\}}\right]  & \leq -\frac{(r-\Vert x \Vert^2)_+}{2\gamma}. 
    \end{align*}
    The lower bound was already shown in Proposition \ref{prop:measure_concentration}. Hence the result.
\end{proof}

\begin{lemma} \label{lm:inside_ball}
    Let $x\in\RR^d$, $\gamma,\delta>0$, $\varepsilon>0$ and $Y_\delta\sim\mathcal N(x,\delta\gamma I_d)$. Let $R>0$. We have
    \begin{equation}
        \limsup_{\delta \downarrow 0}\delta \log \mu_{x,\delta}\left(((\mathcal{M}_x)_\varepsilon)^\complement \cap \mathbb{B}(0,R)\right) \leq -c_{R,\varepsilon}.
    \end{equation}
    Therefore for sufficiently large $R$
    \begin{equation}
        \limsup_{\delta \downarrow 0}\delta \log \mu_{x,\delta}\left(((\mathcal{M}_x)_\varepsilon)^\complement \right) <0.
    \end{equation}
    hence,
    \begin{equation}
        \mu_{x,\delta}\left(((\mathcal{M}_x)_\varepsilon)^\complement \right) \xrightarrow[\delta\downarrow 0]{} 0, \text{ and } \mu_{x,\delta}\left(((\mathcal{M}_x)_\varepsilon)\right) \xrightarrow[\delta\downarrow 0]{} 1
    \end{equation}
\end{lemma}
\begin{proof}
    We have 
    \[
    \mu_{x,\delta}\left(((\mathcal{M}_x)_\varepsilon)^\complement \cap \mathbb{B}(0,R)\right) = \frac{\mathbb{E}\left[e^{-f(T_f(Y_\delta))/\delta} \mathbf{1}_{\left\{ T_f(Y_\delta) \in ((\mathcal{M}_x)_\varepsilon)^\complement \cap \mathbb{B}(0,R)\right\}}\right]}{\mathbb{E}\left[e^{-f(T_f(Y_\delta))/\delta}\right]}
    \]
    and, there exists some set $\mathcal{M}_{x,R,\varepsilon}$ (Fact \ref{fact:1}) such that 
    \[
    \left\{ T_f(Y_\delta) \in ((\mathcal{M}_x)_\varepsilon)^\complement \cap \mathbb{B}(0,R)\right\} \subset \mathcal{M}_{x,R,\varepsilon}.
    \]
    Hence,
    \[
    \mathbb{E}\left[e^{-f(T_f(Y_\delta))/\delta} \mathbf{1}_{\left\{ T_f(Y_\delta) \in ((\mathcal{M}_x)_\varepsilon)^\complement \cap \mathbb{B}(0,R)\right\}}\right] \leq \sum_{M \in \mathcal{M}_{x,R,\varepsilon}} e^{-f(M)/\delta} \mathbb{P}(Y_\delta \in \Att(M)),
    \]
    which is sufficient to conclude for the first inequality by applying Lemma \ref{lm:delta_Fm}. The second inequality comes from the fact that 
    \[
    \mu_{x,\delta}\left(((\mathcal{M}_x)_\varepsilon)^\complement \right) \leq \mu_{x,\delta}\left(((\mathcal{M}_x)_\varepsilon)^\complement \cap \mathbb{B}(0,R) \right) + \mu_{x,\delta}(\mathbb{B}(0,R)^\complement),
    \]
    where the second complement is taken in $\mathcal{M}_\ast$. We apply Proposition \ref{prop:measure_concentration} to obtain the desired result.
\end{proof}
If $R$ is large enough then it means that the mass outside the ball $\mathbb{B}(0,R)$ goes to $0$ when $\delta \downarrow 0$.
\paragraph{Proof of Proposition \ref{prop:barycenter_in_convex_hull}.}
\begin{proof}
Fix $\varepsilon>0$ and $R>0$. For convenience, we write 
\[
m_\delta := \frac{\mathbb{E}[T_f(Y_\delta)e^{-f(T_f(Y_\delta))/\delta}]}{\mathbb{E}[e^{-f(T_f(Y_\delta))/\delta}]}
\]
Since the distance to a closed convex set is a convex function, Jensen's inequality gives
\begin{align*}
\mathrm{dist}\left(m_\delta,\mathrm{conv}(\mathcal{M}_x)\right)
& = \mathrm{dist} \left(\int_{\mathcal{M}_\ast} z\,\mu_{\delta,x}(dz),\mathrm{conv}(\mathcal{M}_x)\right) \\
& \leq
\int_{\mathcal{M}_\ast}\mathrm{dist}\left(z,\mathrm{conv}(\mathcal{M}_x)\right)\,\mu_{\delta,x}(dz).
\end{align*}
We split the integral over the three regions $(\mathcal{M}_x)_\varepsilon\cap \mathbb{B}(0,R)$, $\mathbb{B}(0,R)\setminus (\mathcal{M}_x)_\varepsilon$, and $\mathbb{B}(0,R)^\complement$.

If $z\in (\mathcal{M}_x)_\varepsilon$, then since $\mathcal{M}_x\subset \mathrm{conv}(\mathcal{M}_x)$, $\mathrm{dist}\left(z,\mathrm{conv}(\mathcal{M}_x)\right)\leq \mathrm{dist}(z,\mathcal{M}_x)<\varepsilon.$

If $z\in \mathbb{B}(0,R)\setminus (\mathcal{M}_x)_\varepsilon$, then $\mathrm{dist}\left(z,\mathrm{conv}(\mathcal{M}_x)\right)\leq \|z\|\leq R.$

If $\|z\|>R$, then $\mathrm{dist}\left(z,\mathrm{conv}(\mathcal{M}_x)\right)\leq \|z\|.$
Therefore,
\[
\mathrm{dist}\left(m_\delta,\mathrm{conv}(\mathcal{M}_x)\right)
\leq
\varepsilon
+R\,\mu_{\delta,x}\left(((\mathcal{M}_x)_\varepsilon)^\complement\right)
+\int_{\mathcal{M}_\ast}\|z\|\mathbf 1_{\{\|z\|>R\}}\,\mu_{\delta,x}(dz).
\]

Now let $\delta\downarrow0$. By Lemma \ref{lm:inside_ball} ,
\[
\mu_\delta\left(((\mathcal{M}_x)_\varepsilon)^\complement\right)\to 0,
\]
hence
\[
\limsup_{\delta\downarrow0}\mathrm{dist}\left(m_\delta,\mathrm{conv}(\mathcal{M}_x)\right)
\leq \varepsilon +\limsup_{\delta\downarrow0}
\int_{\mathcal{M}_\ast}\|z\|\mathbf 1_{\{\|z\|>R\}}\,\mu_{\delta,x}(dz).
\]
Next let $R\to\infty$ and use Lemma \ref{prop:gaussian_tail_poly} on the righthandside integral to obtain
\[
\limsup_{\delta\downarrow0}\mathrm{dist}\left(m_\delta,\mathrm{conv}(\mathcal{M}_x)\right)\leq \varepsilon.
\]
Since $\varepsilon>0$ is arbitrary, this proves
\[
\mathrm{dist}\left(m_\delta,\mathrm{conv}(\mathcal{M}_x)\right)\xrightarrow[\delta \downarrow 0]{}0.
\]
The limit in the isolated local minimizer case follows immediately.
\end{proof}

\paragraph{Proof of Theorem \ref{th:exact_expect_pbh}.}
\begin{proof}
    First, by rejecting worse update and coercivity of $f$ (Assumption \ref{ass:f_properties1}), iterates are bounded and stay in some bounded set $B \subset \RR^d$. Thus 
    \[
    d_f := \sup_{x \in B} \mathrm{dist}(x, \overline{\Att(x_\ast)}) <+\infty.
    \]
    Set $\gamma>0$. For every $x\in B$, 
    \[
    V_x(x_\ast) \leq \frac{1}{2\gamma}d_f^2 +f_{\min} 
    \]
    If $\gamma> \frac{d_f^2}{2\mu}$ then, for all $M \neq x_\ast$,
    \[
    V_x(x_\ast) < f_{\min} + \mu \leq V_x(M).
    \]
    Meaning that for $\gamma$ sufficiently large, by Proposition \ref{prop:barycenter_in_convex_hull} there exists $\overline \delta>0$ and $r>0$ such that for all $\delta \leq \overline \delta$
    \[
    \Vert m_{\delta,\gamma} - x_\ast \Vert \leq r.
    \]
    Therefore, if there exists $r>0$, such that $\mathbb{B}(x_\ast,r) \subset \Att(x_\ast)$, then $x_+ = x_\ast$. Such statement holds in our context as $x_\ast$ is an isolated local minimizer, reached by $T_f$. We thus obtain immediate convergence if $\gamma$ is large enough and $\delta$ small enough.
    
    Now there are several scenarios:
    \begin{itemize}
        \item if progress (i.e., strict decrease) is made at every iteration, as iterates are bounded, we visit finitely many $M\in \mathcal{M}$ until convergence (see proof of Theorem \ref{th:ideal_pbh}).
        \item As derived just before, we can only "not improve" the next iterate finitely many times before reaching $\gamma_k$ large enough, hence we obtain convergence as soon as $\delta$ is small enough which happens eventually.
    \end{itemize} 
    Hence, we obtain convergence of the iterates to the global minimizer.
\end{proof}

\subsection{Proofs for approximate expectation PBH} \label{app:proofs_approx_expect_pbh}

\paragraph{Proof of Lemma \ref{lm:well_posed_N}.}
\begin{lemma}\label{lm:well_posed_N} Let $x\in \RR^d$, $\delta, \gamma >0$ and i.i.d. samples $\{y_i\} \sim Y_\delta \sim \mathcal{N}(x,\delta \gamma \Id)$. Then, almost surely,
\begin{equation}
    m_{N,\delta,\gamma}:=\frac{\sum_{i=1}^N T_f(y_i) \exp(-f\left(T_f(y_i) \right)/\delta)}{\sum_{i=1}^N \exp(-f\left(T_f(y_i) \right)/\delta)} \xrightarrow[N\rightarrow+\infty]{} m_{\delta,\gamma}.
\end{equation}
\end{lemma}
\begin{proof}
    We have first that for all $N\in\mathbb{N}$,
    $
    \sum_{i=1}^N \exp(-f\left(T_f(y_i) \right)/\delta)>0,
    $ each term is integrable against the Gaussian measure, 
    and thus by the strong law of large numbers $
    \frac{1}{N}\sum_{i=1}^N \exp(-f\left(T_f(y_i) \right)/\delta) \xrightarrow[N\rightarrow+\infty]{} \mathbb{E}\left[e^{-f(T_f(Y_\delta))/\delta}\right]
    $ a.s..
    On the other end, by Lemma \ref{lm:Tf_bounded}, each term in $\sum_{i=1}^N T_f(y_i) \exp(-f\left(T_f(y_i) \right)/\delta)$ is integrable against the Gaussian measure, and thus by the strong law of large numbers $
    \frac{1}{N}\sum_{i=1}^N T_f(y_i)\exp(-f\left(T_f(y_i) \right)/\delta) \xrightarrow[N\rightarrow+\infty]{} \mathbb{E}\left[T_f(Y_\delta)e^{-f(T_f(Y_\delta))/\delta}\right]
    $ a.s.. Hence, $m_{N,\delta,\gamma} \xrightarrow[N\to +\infty]{} m_{\delta,\gamma}$ a.s..
\end{proof}

\paragraph{Proof of Lemma \ref{lm:m_N_good}.}
\begin{lemma}\label{lm:m_N_good}Let $x\in \RR^d$, $\delta, \gamma >0$ and i.i.d. samples $\{y_i\} \sim Y_\delta \sim \mathcal{N}(x,\delta \gamma \Id)$. 
    Suppose that there exists $M \in \mathcal{M}$ such that $\mathrm{dist}(m_{\delta,\gamma},\partial \Att(M))>0$.
    Then, almost surely,
    \begin{equation}
    T_f(m_{N,\delta,\gamma})\xrightarrow[N\rightarrow+\infty]{} T_f(m_{\delta,\gamma}).
\end{equation}
\end{lemma}

\begin{proof}
    We have by Assumption \ref{ass:Tf_continuity} that $T_f$ is continuous on any interior of $\Att(M)$, hence the result follows from Lemma \ref{lm:well_posed_N}.
\end{proof}

\paragraph{Proof of Lemma \ref{lm:cluster-points-nearby-components}.}
\begin{lemma}\label{lm:cluster-points-nearby-components}
Let $(m_N)_{N\geq 1}$ be a sequence in $\RR^d$ such that
$m_N \rightarrow m, \quad N\rightarrow + \infty$.
Assume that for sufficiently large $N$, $m_N$ does not belong to $\mathcal{N}_s$ the measurable null set not included in the partition of $\RR^d$ defined by $\bigcup_{M \in \mathcal{M}} \Att(M)$ (Lemma \ref{lm:null_set}).
Then every accumulation point of the sequence $(T_f(m_N))_{N\geq 1}$ belongs to
\begin{equation}
\bigcup_{Q\in\mathrm{Adj}(m)} Q,
\qquad
\mathrm{Adj}(m):=\{Q\in\mathcal M:\ m\in \overline{\Att(Q)}\}.
\end{equation}
\end{lemma}
\begin{proof}
Let $\overline m$ be a cluster point of $(T_f(m_N))_{N\geq 1}$. Then there exists a
subsequence $(m_{N_k})_{k\geq 1}$ such that
\[
T_f(m_{N_k})\rightarrow \overline m.
\]
Since $(T_f(m_{N_k}))_{k\geq 1}$ converges, it is bounded. Hence there exists a compact set $K\subset \RR^d$ such that $
T_f(m_{N_k})\in K \text{ for all }k.
$
By assumption, only finitely many components of $\mathcal M$ intersect $K$.
Therefore, after extracting a further subsequence, there exists a component
$Q\in\mathcal M$ such that
\[
T_f(m_{N_{k_\ell}})\in Q \text{ for all }\ell \geq 1.
\]
On the other hand, for $k$ large enough one has $m_{N_k}\notin \mathcal N$, and therefore $m_{N_k}\in \Att(M_k)$ for a unique component $M_k\in\mathcal M$.
Because $T_f(m_{N_k})\in Q$ and $T_f(y)\in M$ whenever $y\in\Att(M)$, we must have $M_k=Q$. Hence for all sufficiently large $k$, $m_{N_k}\in \Att(Q).$
Passing to the limit and using $m_{N_k}\to m$, we obtain $m\in \overline{\Att(Q)}$.
Thus $Q\in \mathrm{Adj}(m)$.
\end{proof}

\paragraph{Proof of Theorem \ref{th:approx_expect_infty_pbh}.}
\begin{proof}
    First, as the number of samples grows to infinity, by Lemma \ref{lm:well_posed_N}, all $m_{N,\delta,\gamma}$ are close to $m_{\delta,\gamma}$, hence we the behavior of the algorithm is governed asymptotically by the behavior of Algorithm \ref{alg:exact_expect_pbh}.

    By rejecting worse update and coercivity of $f$ (Assumption \ref{ass:f_properties1}), iterates are bounded and stay in some bounded set $B \subset \RR^d$. Thus 
    \[
    d_f := \sup_{x \in B} \mathrm{dist}(x, \overline{\Att(x_\ast)}) <+\infty.
    \]
    Set $\gamma>0$. For every $x\in B$, 
    \[
    V_x(x_\ast) \leq \frac{1}{2\gamma}d_f^2 +f_{\min} 
    \]
    If $\gamma> \frac{d_f^2}{2\mu}$ then, for all $M \neq x_\ast$,
    \[
    V_x(x_\ast) < f_{\min} + \mu \leq V_x(M).
    \]
    Meaning that for $\gamma$ sufficiently large, by Proposition \ref{prop:barycenter_in_convex_hull} there exists $\overline \delta>0$ and $r_1>0$ such that for all $\delta \leq \overline \delta$
    \[
    \Vert m_{\delta,\gamma} - x_\ast \Vert \leq r_1.
    \]
    Moreover, for $N$ sufficiently large, there exists $r_2>0$ such that
    \[
    \Vert m_{N,\delta,\gamma} - m_{\delta,\gamma} \Vert \leq r_2.
    \]
    Therefore, as there exists $r>0$, such that $\mathbb{B}(x_\ast,r) \subset \Att(x_\ast)$, and $r_1+r_2\leq r$, then $x_+ = x_\ast$. We thus obtain immediate convergence if $\gamma$ is large enough, $N$ is large enough and $\delta$ small enough.
    
    Now as in Theorem \ref{th:exact_expect_pbh}, there are several scenarios:
    \begin{itemize}
        \item if progress (i.e., strict decrease) is made at every iteration, as iterates are bounded, we visit finitely many $M\in \mathcal{M}$ until convergence (see proof of Theorem \ref{th:ideal_pbh}). We can however visit non minimizing point as $m_{N,\delta,\gamma}$ may sit with non zero probability in the null set of the partition defined by $\Att(M)$ for all $M\in\mathcal M$, hence finite convergence cannot be guaranteed to happen.
        \item As derived just before, we can only "not improve" the next iterate finitely many times before reaching $\gamma_k$ large enough, hence we obtain convergence as soon as $\delta$ is small enough, and $N$ is big enough which happens eventually.
    \end{itemize} 
    Hence, we obtain convergence of the iterates to the global minimizer.
\end{proof}

\paragraph{Proof of Proposition \ref{prop:high_proba_m}.}
\begin{proof}
Denote 
$m_{N,\delta,\gamma}=\frac{A_N}{B_N}$,
$m_{\delta,\gamma}=\frac{A}{B}$,
where
\[
A_N=\frac{1}{N}\sum_{i=1}^N T_f(y_i)\exp(-f(T_f(y_i))/\delta), \text{ and } A=\mathbb E[T_f(Y_\delta)\exp(-f(T_f(Y_\delta))/\delta)],
\] 
while 
\[
B_N=\frac{1}{N}\sum_{i=1}^N \exp(-f(T_f(y_i))/\delta), \text{ and } B=\mathbb E[\exp(-f(T_f(Y_\delta))/\delta)].
\]
We want to control 
\[ 
\left\Vert \frac{A_N}{B_N} - \frac{A}{B} \right\Vert = \left\Vert \frac{B(A_N - A)-A(B_N-B)}{B B_N}\right\Vert,  
\]
w.r.t. quantities that we control, i.e., $\| A_N - A \|$ and $|B_N-B|$. 
We first look at the denominator. Set $0<\eta<B$ and the event $\{|B_N-B|\leq \eta\}$. On this event, we have $B_N\ge B-\eta>0$, implying by the triangular inequality that
\[
\left\Vert m_{N,\delta,\gamma}-m_{\delta,\gamma}\right\Vert = \left\|\frac{B(A_N-A)-A(B_N-B)}{BB_N}\right\| \leq \frac{\|A_N-A\|}{B-\eta} + \frac{\|A\|\,|B_N-B|}{B(B-\eta)}.
\]
Now, to obtain $\|m_{N,\delta,\gamma}-m_{\delta,\gamma}\| \leq \varepsilon,$ we bound the two terms on the righthandside of this inequality w.r.t. $B-\eta$. If $\|A_N-A\|\leq \frac{\varepsilon}{2}(B-\eta)$, and 
$|B_N-B| \leq \min\left(\eta,\frac{\varepsilon B(B-\eta)}{2\|A\|}\right)$,
then we obtain the desired bound. It remains to quantify the probability of these two bounds holding together. Indeed, we have 
\[
\left\{ \|m_{N,\delta,\gamma}-m_{\delta,\gamma}\| \leq \varepsilon\right\} \subset \left\{ \|A_N-A\|\leq \frac{\varepsilon}{2}(B-\eta) \right\} \bigcap \left\{ |B_N-B| \leq \min\left(\eta,\frac{\varepsilon B(B-\eta)}{2\|A\|}\right)\right\}.
\]
The union bound on the complements then tells us that 
\[
\mathbb P \left(\|m_{N,\delta,\gamma}-m_{\delta,\gamma}\|\le \varepsilon\right)
\geq 1 - \mathbb P \left(\|A_N-A\|>\frac{\varepsilon}{2}(B-\eta)\right) - \mathbb P \left(|B_N-B|>\min\left(\eta,\frac{\varepsilon B(B-\eta)}{2\|A\|}\right)\right).
\]
The two righthandside probabilities can be controlled through Markov's and Chebyshev's inequality. We can apply it through the growth control on $f$ and Lemma \ref{lm:Tf_bounded} as $\Vert A_N -A \Vert$ and $|B_N-B|$ are both square integrable against the Gaussian measure. We have
\[
\mathbb P(\| A_N-A \|>t)\leq \frac{ \mathbb E[\| A_N - A \|^2] }{t^2}, \text{ and } \mathbb P(|B_N-B|>s)\leq \frac{\mathrm{Var}(B_N)}{s^2}.
\]
As 
\begin{align*}
    \mathbb E[\| A_N - A \|^2] & = \mathbb E \left[ \frac{1}{N^2} \left\langle \sum_{i=1}^N A^i_N -A, \sum_{i=1}^N A^i_N -A \right\rangle \right] \\
    & = \mathbb E \left[ \frac{1}{N^2} \sum_{i,j=1}^N\left\langle  A^i_N -A, A^j_N -A \right\rangle \right] \\
    & = \frac{1}{N^2} \sum_{i,j=1}^N \mathbb E \left\langle  A^i_N -A, A^j_N -A \right\rangle 
\end{align*}
For $i\neq j$ by independence, we have $\mathbb E \left\langle  A^i_N -A, A^j_N -A \right\rangle = 0$. Remain the $N$, $i=j$ terms, giving 
\[
\mathbb E \left[\|A_N-A\|^2\right] = \frac{1}{N} \mathbb E\|T_f(y_1)\exp(-f(T_f(y_1))/\delta)-A\|^2, 
\] 
while
\[
\mathrm{Var}(B_N)=\frac{\mathrm{Var}(\exp(-f(T_f(y_1))/\delta))}{N},
\]
Therefore, we obtain
\[
\mathbb P \left(\|m_{N,\delta,\gamma}-m_{\delta,\gamma}\| \leq \varepsilon\right)
\geq 1- \frac{4\,\mathbb E\|T_f(y_1)\exp(-f(T_f(y_1))/\delta)-A\|^2}{N\varepsilon^2(B-\eta)^2}
- \frac{\mathrm{Var}(\exp(-f(T_f(y_1))/\delta))}{N\min\left(\eta,\frac{\varepsilon B(B-\eta)}{2\|A\|}\right)^2}.
\]
We deduce then the existence of $\alpha$ by setting $N$ large enough.
\end{proof}
The dependence on $\delta$ may appear highly problematic, but also note that in practice $f(T_f(y_1))$ should be of the order of $\delta$ for our best samples.
\paragraph{Proof of Theorem \ref{th:approx_expect_finite_pbh}.}
\begin{proof}
    The number of samples is finite, by Lemma \ref{lm:well_posed_N}, all $m_{N,\delta,\gamma}$ are close to $m_{\delta,\gamma}$, hence we the behavior of the algorithm is governed asymptotically by the behavior of Algorithm \ref{alg:exact_expect_pbh}.

    By rejecting worse update and coercivity of $f$ (Assumption \ref{ass:f_properties1}), iterates are bounded and stay in some bounded set $B \subset \RR^d$. Thus 
    \[
    d_f := \sup_{x \in B} \mathrm{dist}(x, \overline{\Att(x_\ast)}) <+\infty.
    \]
    Set $\gamma>0$. For every $x\in B$, 
    \[
    V_x(x_\ast) \leq \frac{1}{2\gamma}d_f^2 +f_{\min} 
    \]
    If $\gamma> \frac{d_f^2}{2\mu}$ then, for all $M \neq x_\ast$,
    \[
    V_x(x_\ast) < f_{\min} + \mu \leq V_x(M).
    \]
    Meaning that for $\gamma$ sufficiently large, by Proposition \ref{prop:barycenter_in_convex_hull} there exists $\overline \delta>0$ and $r_1>0$ such that for all $\delta \leq \overline \delta$
    \[
    \Vert m_{\delta,\gamma} - x_\ast \Vert \leq r_1.
    \]
    For $N$ sufficiently large (Proposition \ref{prop:high_proba_m}), there exists $\varepsilon>0$ and $0<\alpha<1$ such that
    \[
    \mathbb P \left(\Vert m_{N,\delta,\gamma} - m_{\delta,\gamma} \Vert \leq \varepsilon \right) \geq \alpha.
    \]
    Therefore, as there exists $r>0$, such that $\mathbb{B}(x_\ast,r) \subset \Att(x_\ast)$, we need $\varepsilon$ small enough so that and $r_1+\varepsilon\leq r$, then with probability superior or equal to $\alpha$, $x_+ = x_\ast$. Thus, if $N$ is large enough so this $\varepsilon$ is small enough, then the high probability follows.%
    
    The rest of the proof is the same as in Theorem \ref{th:approx_expect_infty_pbh}.
    Hence, we obtain convergence of the iterates to the global minimizer with high probability.
\end{proof}
The trick here is that we do not need an uniform control on the probability on all the iterates. We simply need to wait for $\gamma$ to be sufficiently large so that one of the iterates theoretically has $x_\ast$ as one of its minimizers of the potential $V_x$, then if we have enough samples we reach the correct solution with high probability. 

\section{Additional results}
We display additional results for the full scaling law (all parameters available to be fitted \cite{shukorscaling}). The results are worse (in terms of prediction) than those of the additive law, in average. It was expected given the results of \cite{shukorscaling}. Our algorithm stil outperforms the others.
\begin{table}[ht]
\centering
\label{tab:full_pbh_mre_by_samples_lbfgs10}

\small
\setlength{\tabcolsep}{2.6pt}
\begin{tabular}{|ll|ccccccc|cc|}
\hline
Family & Domain & $N=5$ & $N=10$ & $N=15$ & $N=20$ & $N=30$ & $N=40$ & $N=50$ & ZOP & BH \\
\hline
\multirow{7}{*}{\shortstack{LLM\\($d=47$)}}
& arxiv         & 12.085 & 12.303 & \textbf{11.681} & 12.358 & 11.861 & 12.177 & 12.694 & 12.646 & 12.379 \\
& book          & \textbf{10.046} & 10.949 & 10.173 & 10.722 & 10.718 & 10.264 & 10.265  & 10.800 & 10.649 \\
& c4            & 10.937 & 10.949 & 10.874 & 10.936 & 11.115 & 11.472 & \textbf{10.716}  & 11.309 & 11.314 \\
& commoncrawl   & \textbf{10.445} & 11.038 & 11.834 & 11.710 & 11.450 & 11.516 & 11.256  & 11.677 & 11.464 \\
& github        & 17.323 & 17.385 & 18.258 & \textbf{17.305} & 18.434 & 17.899 & 18.470  & 19.882 & 18.847 \\
& stackexchange & \textbf{12.488} & 12.891 & 12.943 & 13.162 & 13.110 & 13.024 & 13.073  & 13.620 & 13.374 \\
& wikipedia     & \textbf{19.658} & 19.859 & 20.044 & 19.664 & 19.915 & 19.739 & 19.749  & 20.529 & 19.827 \\
\hline
\multirow{4}{*}{\shortstack{LVM\\($d=29$)}}
& alttext       & \textbf{4.117} & 6.703 & 11.270 & 10.317 & 12.237 & 14.245 & 14.062 & 13.176 & 10.130 \\
& highquality1  & \textbf{7.634} & 7.758 & 9.393 & 12.712 & 18.447 & 16.315 & 19.492 & 21.979 & 11.756 \\
& highquality2  & \textbf{7.536} & 10.969 & 9.483 & 11.711 & 15.700 & 18.769 & 19.177  & 21.367 & 12.207 \\
& synthetic     & 11.435 & 14.736 & \textbf{11.070} & 13.664 & 14.441 & 16.844 & 21.105 & 29.138 & 15.884 \\
\hline
\multirow{3}{*}{\shortstack{NMM\\($d=23$)}}
& captions      & \textbf{10.944} & 12.190 & 11.921 & 12.111 & 12.098 & 12.596 & 12.347  & 11.351 & 11.317 \\
& interleaved   & 10.839 & \textbf{8.768} & 10.925 & 11.097 & 10.969 & 10.586 & 10.706  & 10.178 & 10.907 \\
& text          & \textbf{6.565} & 9.392 & 8.811 & 9.057 & 9.313 & 9.174 & 9.131 & 8.774 & 8.120 \\
\hline
\end{tabular}
\caption{Test \(
    \mathrm{MRE}
    =
    \frac{|\text{prediction} - \text{observation}|}{\text{observation}}\) (\%) for fitting full scaling law. 
(\texttt{lbfgs steps=10}, \texttt{CPU time=300s}). 
We set $\gamma = 10$ and $\delta = 10^{-3}$ at initialization. Results averaged over three seeds.
} 
\end{table}
\begin{table}[ht]
\centering
\small
\setlength{\tabcolsep}{2.6pt}
\begin{tabular}{|ll|ccccccc|cc|}
\hline
Family & Domain & $N=5$ & $N=10$ & $N=15$ & $N=20$ & $N=30$ & $N=40$ & $N=50$ & ZOP & BH \\
\hline
\multirow{7}{*}{\shortstack{LLM\\($d=47$)}}
& arxiv         & 12.373 & 12.445 & \textbf{11.820} & 12.478 & 12.318 & 12.561 & 12.387  & 12.646 & 12.269 \\
& book          & \textbf{9.910} & 10.947 & 10.529 & 10.590 & 10.451 & 10.295 & 10.256 & 10.800 & 10.629 \\
& c4            & 11.226 & 11.266 & 11.571 & 11.161 & 11.201 & 11.203 & \textbf{10.992}  & 11.309 & 11.060 \\
& commoncrawl   & \textbf{10.650} & 11.254 & 11.733 & 11.249 & 11.321 & 11.679 & 11.331  & 11.677 & 11.474 \\
& github        & \textbf{16.579} & 18.328 & 18.947 & 16.835 & 18.686 & 18.750 & 18.329  & 19.882 & 18.243 \\
& stackexchange & \textbf{11.735} & 12.606 & 13.342 & 12.858 & 13.802 & 13.620 & 13.517  & 13.620 & 13.339 \\
& wikipedia     & \textbf{19.285} & 19.407 & 19.909 & 19.398 & 19.976 & 20.077 & 19.930  & 20.529 & 19.744 \\
\hline
\multirow{4}{*}{\shortstack{LVM\\($d=29$)}}
& alttext       & \textbf{4.438} & 6.761 & 13.017 & 10.932 & 13.403 & 14.336 & 14.217  & 13.176 & 10.094 \\
& highquality1  & \textbf{7.228} & 7.850 & 10.172 & 12.248 & 18.583 & 17.110 & 19.230  & 21.979 & 11.735 \\
& highquality2  & \textbf{7.305} & 10.519 & 9.581 & 12.259 & 17.326 & 18.470 & 19.405  & 21.367 & 12.217 \\
& synthetic     & \textbf{9.318} & 11.464 & 10.898 & 15.507 & 18.410 & 17.687 & 20.299  & 29.138 & 18.764 \\
\hline
\multirow{3}{*}{\shortstack{NMM\\($d=23$)}}
& captions      & \textbf{10.727} & 11.847 & 11.867 & 12.172 & 12.224 & 12.383 & 12.094  & 11.351 & 11.308 \\
& interleaved   & 10.561 & \textbf{8.522} & 10.151 & 11.040 & 10.564 & 10.552 & 10.649  & 10.178 & 10.906 \\
& text          & \textbf{7.202} & 8.772 & 8.902 & 9.023 & 9.787 & 9.147 & 9.147 & 8.774 & 8.120 \\
\hline
\end{tabular}
\caption{Test \(
    \mathrm{MRE}
    =
    \frac{|\text{prediction} - \text{observation}|}{\text{observation}}\) (\%) for fitting full scaling law. 
(\texttt{lbfgs steps=5}, \texttt{CPU time=300s}). 
We set $\gamma = 10$ and $\delta = 10^{-3}$ at initialization. Results averaged over three seeds.
}
\label{tab:full_pbh_mre_by_samples}
\end{table}
\clearpage

\end{document}